%% file: main.tex
\documentclass[11pt]{article}

\usepackage[numbers,sort&compress]{natbib}
\usepackage{geometry}
\usepackage{multirow}
\usepackage[utf8]{inputenc} 
\usepackage[T1]{fontenc}    
\usepackage{hyperref}       
\usepackage{url}            
\usepackage{booktabs}       
\usepackage{amsfonts}       
\usepackage{nicefrac}       
\usepackage{microtype}      
\usepackage{xcolor}         
\usepackage{graphicx}
\usepackage{makecell}

\usepackage{amsmath,amsthm,amssymb}
\makeatletter
\def\@fnsymbol#1{\ensuremath{\ifcase#1\or \dagger\or \ddagger\or
   \mathsection\or \mathparagraph\or \|\or **\or \dagger\dagger
   \or \ddagger\ddagger \else\@ctrerr\fi}}
\makeatother
\newcommand*\samethanks[1][\value{footnote}]{\footnotemark[#1]}

\title{Efficient Online Conformal Selection with Limited Feedback}

\author{Sreenivas Gollapudi\thanks{Google Research, Mountain View, CA, USA. Email:{\tt \{sgollapu,kostaskollias,asinop\}@google.com}} \and Kostas Kollias\samethanks[1] \and Kamesh Munagala\thanks{Department of Computer Science, Duke University, Durham, NC, USA. Email: {\tt kamesh@cs.duke.edu}. This work was done while the author was visiting Google Research.} \and Ali Sinop\samethanks[1] 
}
\date{}

\newtheorem{theorem}{Theorem}
\newtheorem{lemma}{Lemma}
\newtheorem{claim}{Claim}
\newtheorem{corollary}{Corollary}

\newtheorem{example}{Example}

\begin{document}

\maketitle
\input{abstract}
\input{intro_new}

\input{ICLR/full_new}

\input{ICLR/experiment}

\input{ICLR/conclusion}

\paragraph{Acknowledgments.} We thank Aravindan Vijayaraghavan for several helpful discussions. 
The authors used Gemini 3.1 Pro for the following tasks: (1) identifying  and summarizing prior work; (2) paraphrasing and polishing the text; and (3) generating Python code for the experiments. The authors also benefited from  discussions with the model, which helped suggest approaches to proving Theorem~\ref{thm:bandit2}. All AI-generated content, particularly citations, code, and analysis, was verified by the authors, who take full responsibility for its correctness.

\bibliographystyle{plainnat}
\bibliography{references}

\appendix
\input{neurips/bandit}

\end{document}

%% file: abstract.tex

\begin{abstract}
We address the problem of conformal selection, where an agent must select a low-cost subset of options to ensure that at least one ``success'' is identified at a pre-specified target rate $\phi$. While traditional online conformal prediction focuses on maintaining validity for the observed sequence, minimizing the resource cost (efficiency) of such selections---especially under limited feedback---remains a significant challenge. In this work, we consider highly restricted ``bandit'' feedback, where the agent only observes feedback about the subset it selected, and not the true label, point, or outcomes of unchosen options. We demonstrate that the simple Adaptive Conformal Inference (ACI) update rule, when applied to the appropriate control parameter or dual variable and paired with explicit boundary actions, is both adversarially valid, ensuring the success target is met on average for any input sequence (and hence under distribution shifts), and stochastically efficient, achieving sublinear efficiency regret for i.i.d. inputs against an optimal stochastic benchmark. The key algorithmic idea is to avoid the projected updates standard in constrained bandits: projections break the exact telescoping identity behind ACI validity, whereas boundary actions stabilize the unprojected update through actual decisions. We show these guarantees under canonical models capturing bandit feedback via a unified algorithmic technique and Lyapunov-based analysis. Our approach handles more general settings than prior work, while requiring significantly less feedback, and provides a new theoretical bridge between efficient online learning with limited feedback and distribution-free uncertainty quantification.
\end{abstract}

%% file: intro_new.tex
\section{Introduction}

In many real-world decision-making scenarios, an agent must select a subset of options---such as probing edges in a routing network, selecting candidates for an interview, or activating sensors---to ensure the existence of at least one ``success'' (e.g., a valid route, a qualified hire, or a detected signal). This problem, which we term {\em conformal selection}, generalizes traditional conformal prediction from the task of producing uncertainty sets for labels~\cite{vovk2005algorithmic,angelopoulos2021gentle} to the task of selecting optimal subsets under resource constraints.

This work studies sequential conformal selection under highly restricted bandit feedback. At each time $t$, the learner chooses a selection rule, or equivalently a subset, and observes only whether the chosen set contained a success. In our main  model where these selection rules correspond to bandit arms (or actions), the learner also observes the realized cost of the chosen set; it does not observe the success or cost of unchosen actions, the ground-truth label or point, or the scores of the underlying options. This is substantially weaker feedback than recent online conformal literature. (See Table~\ref{tab:comparison}.) We assume a target success level $\phi \in (0,1)$, and two boundary actions: a guaranteed-success action, corresponding to ``choosing everything,'' and a guaranteed-failure action, corresponding to ``choosing nothing.'' These boundary actions  are typically available in the online conformal selection applications like interval selection~\cite{Srinivas26} and scoring rules~\cite{ge2025stochastic,angelopoulos2021gentle} that we describe in Section~\ref{sec:discrete}.

Our goal is to obtain two guarantees simultaneously. The first is {\em adversarial validity}~\cite{GibbsC}: for every realized input sequence, the empirical success rate of the selected sets tracks the target $\phi$. That is, if $Y_t$ denotes whether the set chosen at time $t$ succeeds, then
\[
    \left| \frac{1}{T} \sum_{t=1}^T Y_t - \phi \right| = O\left(T^{-\alpha}\right)
\]
for constant $\alpha > 0$.\footnote{Standard conformal calibration would only require the average success rate to be asymptotically lower bounded by $\phi$; our algorithms obtain asymptotic tracking of $\phi$ as a consequence of the ACI update.}
The second is {\em stochastic efficiency}: when input rewards and costs (or contexts) are drawn i.i.d. from an unknown distribution, the expected cost of the chosen sets has sublinear (in time) regret relative to an optimal distribution-aware baseline that enforces validity in expectation. 

\paragraph{Validity notion.} We remark that our validity notion is orthogonal to the stochastic per-step validity guarantees studied in recent online conformal work~\cite{angelopoulos24a,Srinivas26,ge2025stochastic,WangQ} with stronger full-information or semi-bandit feedback. Our adversarial guarantee is pathwise and long-run; following~\cite{GibbsC,ramalingam2025}, it makes no distributional assumption and remains meaningful under distribution shift. Stochastic per-step validity, by contrast, controls the marginal coverage at each time under i.i.d. or exchangeability assumptions, but is not a pathwise statement. Neither notion implies the other. The point of this work is therefore not to replace stochastic validity, but to show that pathwise validity can be combined with stochastic efficiency even under substantially weaker feedback than in prior work.

\paragraph{Algorithmic Idea.} A natural starting point for minimizing cost subject to a success constraint is the constrained stochastic bandit literature~\cite{AgarwalD,Cayci_Zheng_Eryilmaz_2022,pmlr-v247-guo24a,Slivkins_context}. These algorithms typically introduce a dual variable: the penalty increases after failures, forcing the learner to choose safer, costlier actions, and decreases after successes. This is similar to the Adaptive Conformal Inference (ACI) update~\cite{GibbsC} for ensuring validity, which given a control parameter $\kappa_t$ (say a dual variable),  has the form
\[
    \kappa_{t+1} = \kappa_t + \eta(\phi - Y_t).
\]
In classical stochastic bandit literature, to keep the resulting dual variable bounded, standard algorithms use a projection, floor, or cap. This is benign for stochastic regret, but conflicts with adversarial conformal validity. The ACI update obtains pathwise (adversarial) conformal validity through an exact telescoping identity: the control parameter keeps a running sum of past deviations from the target. Projecting this parameter causes the algorithm to ``forget'' uncompensated errors, introducing residual terms that can be linear on adversarial sequences.

We provide a simple resolution to this incompatibility. Instead of projecting the ACI update, we leave it unprojected and stabilize it through physical boundary actions.  When the control parameter becomes too large, the algorithm plays the guaranteed-success action; when it becomes too small, it plays the guaranteed-failure action. Thus the state is stabilized by actual decisions rather than by modifying the update. This preserves the exact telescoping identity needed for adversarial validity, while still allowing a Lyapunov-drift analysis for stochastic efficiency.

\subsection{Our Contributions}

Our main contribution in Section~\ref{sec:full} is a multi-armed bandit framework for efficient online conformal selection with pathwise validity. Our results can be summarized as follows.

\paragraph{\bf Finite-arm conformal selection with bandit feedback.}
In Section~\ref{sec:full}, we present our main result. We formulate conformal selection as an $n$-armed bandit problem. Each arm is a possible selection rule, and the learner observes only the success and cost of the arm it plays (that is, of the resulting chosen set). This model covers scalar score thresholds, multi-dimensional or unstructured selection spaces, and  other settings where no reliable scalar score is available. Because there may be no continuous calibration parameter, we apply the ACI update to a dual variable. The key algorithmic step is to use the explicit boundary arms described above instead of the projected dual updates standard in constrained bandits. We prove pathwise adversarial validity, and in the i.i.d. setting we show stochastic efficiency regret
$
    \widetilde{O}\!\left(\sqrt{nT\log(nT)}\right)
$
against the optimal stochastic policy satisfying the expected validity constraint.

\paragraph{\bf Consequences for score thresholds and interval selection.}
In Section~\ref{sec:discrete}, we instantiate our main theorem in two canonical online conformal selection settings considered (with more feedback) in recent work. First, for a fixed scalar scoring rule~\cite{ge2025stochastic,WangQ,yang2026onlineconformalpredictionadversarial,angelopoulos2021gentle} where the learner only observes success/failure of the chosen threshold along with cost of the chosen set (but not the true label), discretizing the threshold space reduces the problem to the bandit model above. Under a Lipschitz assumption on the expected cost, this yields $\widetilde{O}(T^{2/3})$ stochastic efficiency regret while retaining adversarial validity. Second, for online interval selection on $[0,1]$~\cite{Srinivas26}, discretizing the space of intervals yields $\widetilde{O}(T^{3/4})$ stochastic efficiency regret under bandit feedback, where the learner observes only whether the selected interval contains the point and not the point itself. These corollaries show that the finite-arm result implies conformal threshold calibration under bandit feedback, while also applying to selection problems without a scalar score. We compare these results to prior work in terms of feedback and validity models, efficiency benchmark, and regret bounds in Table~\ref{tab:comparison}.

\paragraph{Additional results.} We also give, in  Appendix~\ref{sec:bandit}, a lightweight success-only scalar-threshold variant that trades off generality and regret rate for lower feedback requirements, again presenting a comparison in Table~\ref{tab:comparison}.

\begin{table*}[t]
\centering
\caption{Comparison with related online conformal prediction literature. Our feedback notion is weaker than prior work, while the validity notions are different. }
\label{tab:comparison}
\resizebox{\textwidth}{!}{%
\small
\begin{tabular}{l l l l l l l}
\toprule
\textbf{Setting} & \textbf{Prior work} & \textbf{This paper} & \textbf{Feedback} & \textbf{Validity notion} & \textbf{Efficiency benchmark} & \textbf{Efficiency regret} \\
\midrule
\multirow{3}{*}{\makecell{Scoring function \\ (Section~\ref{sec:fixed})}}
& \cite{ge2025stochastic,WangQ}
& --
& Semi-bandit (sees true label)
& Stochastic
& Optimal stochastic policy
& $\widetilde{O}(\sqrt{T})$ \\
& --
& Section~\ref{sec:full}
& Bandit (sees cost and reward)
& Adversarial
& Optimal stochastic policy\footnotemark[1]
& $\widetilde{O}(T^{2/3})$ \\ 
& --
& Appendix~\ref{sec:bandit}
& Bandit (does not see cost)
& Adversarial
& Optimal fixed-threshold policy\footnotemark[2]
& $\widetilde{O}(T^{3/4})$ \\
\midrule
\multirow{2}{*}{\makecell{1-D intervals \\ (Section~\ref{sec:interval_instantiation})}}
& \cite{Srinivas26}
& --
& Full info. (sees data point)
& Stochastic
& Optimal stochastic policy\footnotemark[3] 
& $\widetilde{O}(\sqrt{T})$\\
& --
& Section~\ref{sec:full}
& Bandit (doesn't see data point)
& Adversarial
& Optimal stochastic policy
& $\widetilde{O}(T^{3/4})$ \\
\bottomrule
\end{tabular}%
}
\end{table*}

\footnotetext[1]{Assumes expected cost is Lipschitz (see Section~\ref{sec:discrete}).}
\footnotetext[2]{Assumes ``non-flatness'' of expected reward and Lipschitzness on expected cost (see Appendix~\ref{sec:bandit}).}


\subsection{Related Work}
\label{sec:related}

Conformal prediction (CP) was originally developed as a framework for providing distribution-free uncertainty quantification with finite-sample coverage guarantees~\cite{vovk2005algorithmic,angelopoulos2021gentle}. While traditional CP relies on exchangeability, recent work has studied online settings with distribution shift or adversarial behavior~\cite{GibbsC,Gupta22,Feldman22}. In this setting, ACI~\cite{GibbsC} provides a robust mechanism for pathwise validity by updating a calibration parameter so that the average success rate tracks the target $\phi$. However, validity alone does not address efficiency: one also wants the selected sets to have small cost. Efficient policies are known for offline conformal prediction~\cite{gao2025volume}, and low-regret online policies are known under full-information feedback (1D interval selection~\cite{Srinivas26}) and semi-bandit feedback (scoring rules~\cite{ge2025stochastic,WangQ,yang2026onlineconformalpredictionadversarial}). These approaches do not directly apply under the weaker full bandit feedback model considered here, where the learner probes a set and only observes whether the set succeeded, without observing the ground-truth label or point.\footnote{The work of~\cite{WangQ} refers to ``bandit feedback'', but their feedback model is closer to semi-bandit feedback: the learner probes one label and learns whether it is true, which is a special case of the model in~\cite{ge2025stochastic}. In contrast, we probe a set and only learn whether the label was contained in the set.} 

As mentioned before, the validity notion we study is also different from the stochastic per-step validity guarantees in the recent online CP literature with stronger feedback~\cite{angelopoulos24a,Srinivas26,ge2025stochastic,WangQ}. Our results are therefore not a replacement for stochastic validity guarantees, but rather show that pathwise ACI-style validity can be combined with stochastic efficiency under substantially weaker feedback. In particular, the validity guarantee is adversarial and sequence-wise, while the efficiency guarantee is stochastic and benchmarked against an optimal distribution-aware policy.

This combination is related in spirit to best-of-both-worlds guarantees, but differs from the usual formulation. When the distribution of inputs can shift over time, lower bounds show that no sublinear regret is possible for adversarial bandits with knapsacks~\cite{Immorlica22}; analogous lower bounds have been established for online conformal prediction even with full-information feedback~\cite{Srinivas26}.\footnote{The contemporaneous work of~\cite{yang2026onlineconformalpredictionadversarial} side-steps this impossibility by scalarizing miscoverage and inefficiency into a single loss function, while scaling the inefficiency penalty by $O(T^{-1})$. Consequently, unlike our work, they do not provide separate sublinear regret bounds for adversarial validity and stochastic efficiency relative to an optimal valid baseline.} We therefore study the intermediate regime of adversarial validity and stochastic efficiency. This is also distinct from recent decaying-step-size ACI results~\cite{angelopoulos24a,areces2025online}, which provide best-of-both-worlds guarantees for the granularity of coverage, and from best-of-both-worlds bandit algorithms~\cite{Zimmert19a,ZimmertL}, which seek simultaneous guarantees for the same performance measure.

Our bandit formulation is related to discretization approaches in dynamic pricing~\cite{KleinbergL}. When conformal selection is parameterized by a scalar score threshold, one can discretize the calibration parameter and treat each discretized threshold as an arm, much as a pricing algorithm discretizes possible prices. For dynamic pricing, $\Omega(T^{2/3})$ regret lower bounds are known.  Conformal selection differs from dynamic pricing and reward-maximization bandits because each arm plays two roles: it produces a success indicator used for validity and incurs a resource cost used for efficiency. The goal is therefore not simply to maximize reward, but to simultaneously maintain pathwise validity and minimize stochastic cost.

This makes the closest algorithmic connection to primal-dual stochastic optimization, bandits with knapsacks, and constrained bandits~\cite{BwK,AgarwalD,ding2020,Cayci_Zheng_Eryilmaz_2022,pmlr-v247-guo24a,Slivkins_context}. These methods typically stabilize dual variables through confidence-adjusted estimates, smoothed updates, projections, floors, or feasibility assumptions such as Slater conditions. Such tools are useful for stochastic regret, but  as mentioned before, they are incompatible with ACI-style adversarial validity. Our main technical departure is to keep the ACI update unprojected and stabilize it through explicit boundary actions, which may be suboptimal in the short run but preserve the validity identity. The analysis then adapts Lyapunov-drift arguments from stochastic control and constrained bandits~\cite{Neely,kushner2013stochastic,Cayci_Zheng_Eryilmaz_2022,pmlr-v247-guo24a} to this unprojected, boundary-stabilized system. This also distinguishes our setting from primal-dual online convex optimization with long-term constraints~\cite{mahdavi2012trading}, which assumes full-information feedback and typically uses projected dual updates.

Finally, our algorithms connect to online learning interpretations of ACI. As observed in~\cite{ramalingam2025,angelopoulos24a,angelopoulos2025gradient}, ACI can be viewed as online subgradient descent on the pinball loss, and is also closely related to Robbins--Monro stochastic approximation~\cite{kushner2013stochastic}. Prior work has used this perspective primarily to prove coverage tracking or convergence to a population quantile. We show that the same unprojected update, when applied to an appropriate primal or dual control parameter, can also control resource consumption and yield stochastic efficiency regret bounds under bandit feedback. Related work on conformal prediction under non-exchangeability and time-series shift~\cite{tibshirani2019conformal,barber2023conformal,angelopoulos2023conformal,angelopoulos2025gradient} primarily studies coverage or gradient-tracking guarantees under fuller feedback; incorporating PID-style controllers~\cite{angelopoulos2023conformal} into our updates is an interesting direction for  improving empirical efficiency.

%% file: ICLR/full_new.tex
\section{Online Conformal Selection with Bandit Feedback}
\label{sec:full}

We now present our main result: a general multi-armed bandit model for conformal selection, and the boundary-stabilized primal-dual algorithm for it. 
To keep the model concrete, consider a ride-sharing platform where, given an input request $x_t$, the goal is to find a subset $\mathcal C(x_t)$ of drivers to probe for availability before passing the result to a downstream matching algorithm. Probing more drivers increases the chance that the request is successfully matched, but incurs cost by increasing notification fatigue and marketplace disruption. The calibrating policy may only observe whether some probed driver was eventually matched to the request, and not the identity of the driver or the drivers that would have accepted. A finite collection of possible probing rules, for instance rules based on distance cutoffs, driver-quality cutoffs, or score thresholds, can be treated as bandit arms. This is the basic abstraction we study below. We return to scalar score functions, including this ride-sharing example, and to other instantiations such as multi-dimensional scoring rules and interval selection in Section~\ref{sec:discrete}.

\subsection{The Multi-armed Bandit Model}
\label{sec:mab_model}

We will adapt the multi-armed bandit framework for conformal selection as follows.  We consider a setting with $n$ arms. At each epoch $t$, a success vector $\mathbf{r}_t \in \{0, 1\}^n$ and cost vector $\mathbf{c}_t \in [0, C_{\max}]^n$ are realized, where for simplicity of algebra (and w.l.o.g.), we assume $C_{\max} \ge 1$. We further make the standard assumption that $n = o(T)$. The system plays one arm $i_t$, and observes only the reward $r_{i_t,t}$ and cost $c_{i_t,t}$ of arm $i_t$. This is therefore the standard bandit feedback in literature: the learner only observes the success and cost of the chosen selection rule, and does not observe the full success vector, the full cost vector, or the underlying label/point that generated the feedback.

Our goal is to guarantee adversarial conformal validity while having low stochastic efficiency regret. Formally, we seek the following.

\begin{description}
\item[Conformal Validity:] For a conformal guarantee $\phi \in (0,1)$, we seek a finite sample success rate $ \frac{1}{T} \sum_{t=1}^T  Y_t \rightarrow \phi$ regardless of distributional assumptions on the input sequence. Here $Y_t = r_{i_t,t}$ is the success indicator of the arm actually played at time $t$.
\item[Efficiency of Probing Budget:] In the stochastic setting where the reward vectors $\mathbf{r}_t$ and cost vectors $\mathbf{c}_t$ are drawn i.i.d. from an unknown distribution $\mathcal D$, the strategy should have low regret (on cost) against the optimal policy that minimizes the expected per step cost subject to the constraint that the expected per-step validity is at least $\phi$. We present this benchmark formally in Section~\ref{sec:full_stoch}.
\end{description}

Our {\bf main assumption} is the existence of an arm $i_{\max}$ such that $c_{i_{\max},t} = C_{\max}$ and $r_{i_{\max},t} = 1$ always, and the existence of an arm $i_{\min}$ with $c_{i_{\min},t} = 0$ and $r_{i_{\min},t} = 0$ always; these correspond intuitively to ``choosing everything'' and ``choosing nothing'' respectively. This assumption is natural for the applications studied in Section~\ref{sec:discrete}. For instance, in interval selection, typically, the empty interval and the entire range are available; in scoring rules, the score can be very strict or very lenient, and so on. Finally, we assume $\phi$ is bounded away from both $0$ and $1$ by a constant, which is also a natural assumption for most applications.

\input{ICLR/algorithm}

\input{ICLR/pdproof}

\subsection{Instantiations: Score Thresholds and Intervals}
\label{sec:discrete}

We now show how the finite-arm model above captures two canonical online conformal selection settings studied in recent work. The first is a scalar-threshold scoring-rule setting~\cite{ge2025stochastic,WangQ}, which is the closest analogue of standard conformal calibration with a fixed score function, and its multi-dimensional generalization. The second is online interval selection~\cite{Srinivas26}, which does not require a score function. For simplicity, we focus on the dependence of the regret bound on $T$, omitting the other factors that we treat as constants.

\subsubsection{Fixed Scoring Rules with a Scalar Threshold}
\label{sec:fixed}

We first consider a simple conformal setting that generalizes calibration based on a fixed score function. In this setting, we have a black-box algorithm {\sc ALG}$(x,\tau)$ that takes as input a context $x \in \mathcal{X}$ and a scalar $\tau$. It outputs a success indicator $Y(x,\tau)$, so that either $Y(x,\tau) = 0$ (failure) or $Y(x,\tau) = 1$ (success). We assume this function is monotone in $\tau$ for any fixed $x$, so that there exists $\tau_x \in (\tau_{\min},\tau_{\max})$ such that $Y(x,\tau) = 0$ for $\tau < \tau_x$, and $Y(x,\tau) = 1$ for $\tau \ge \tau_x$. Let $Q = \tau_{\max} - \tau_{\min}$. There is a cost function $C$, so that the algorithm spends cost budget $C(x, \tau)$ when given input context $x$ and parameter $\tau$. We assume that for any fixed $x \in \mathcal X$, the function $C(x, \tau)$ is monotonically non-decreasing in  $\tau$, so that increasing $\tau$ is more costly, but is more likely to result in success. 

There is an online sequence of contexts $x_t$. At each step $t$, we need to choose a parameter $\tau_t$ and run {\sc ALG}$(x_t, \tau_t)$, where $x_t$ itself could be hidden from the calibrating policy; the policy obtains $Y(x_t,\tau_t)$ as feedback. In the finite-arm reduction in this section, the policy also observes the realized cost $C(x_t,\tau_t)$ of the chosen threshold. In Appendix~\ref{sec:bandit}, we analyze a simpler continuous controller that only needs the success feedback and not the cost feedback.

We can view {\sc ALG} as implementing a given calibrated score function, though we have framed it more abstractly. It is easy to check that this setting generalizes the usual classification setting, where at each step, the input is a classification task $(x_t, y_t)$; we observe the instance $x_t$ and seek to find a label set $\mathcal C(x_t)$ that contains the ground truth label $y_t$. There is a non-conformity score function $f(x,y)$, such that a lower score implies $y$ is more likely to be the correct label for $x$. For a parameter $\tau_t$, the conformal routine {\sc ALG} sets $\mathcal C(x_t) = \{ y \mid f(x_t,y) \le \tau_t \}$ and outputs this set. It gets feedback on whether $y_t \in \mathcal C(x_t)$, but not the identity of $y_t$. The cost paid by {\sc ALG} at step $t$ is $|\mathcal C(x_t)|$, the size of the conformal set.

The work of~\cite{ge2025stochastic} considers the same problem with semi-bandit feedback where the ground-truth classification becomes known if it lies in the calibration set; similarly, the work of~\cite{WangQ} picks one label instead of a set of labels. In contrast, our setting simply assumes a black-box algorithm that only provides bandit feedback on whether the ground truth lay within the calibration set, without revealing the corresponding label. Further, we enforce the requirement of worst-case per-sequence calibration~\cite{GibbsC}, as opposed to calibration under an exchangeability assumption on the input sequence.

Our model also encompasses other notions of scoring, for instance, the ride-sharing setting below. 

\begin{example} [Ride-share with a Scoring Function]
\label{eg:ride1}
Consider a ride-sharing application where given an input request $x_t$, the goal is to find a subset $\mathcal C(x_t)$ of drivers to probe and determine availability that can be input to a downstream matching algorithm. Each driver $j$ is mapped to a feature vector $y_j$ (for instance, ETA to the request, number of hours driven, etc.) and is assigned a score $g(x_t; y_j)$ that captures the quality of the match (for instance, the match probability). Given a parameter $\tau_t$, {\sc ALG} chooses $\mathcal C(x_t) = \arg\min_{S} \left\{ |S| \mid \sum_{j \in S} g(x_t; y_j) \ge \tau_t \right\}$, which corresponds to probing those drivers with largest match quality subject to the total quality exceeding threshold $\tau_t$. (If the $\arg\min$ is undefined, we set $S$ to be the set of all drivers.) The cost is $|\mathcal C(x_t)|$, the number of drivers probed, which must be kept small to prevent notification fatigue. Our probing algorithm works with the most limited feedback of whether some driver in $\mathcal C(x_t)$ was eventually matched to the request, but not the identity of the driver, which may be privy to the downstream matching algorithm.
\end{example}

\paragraph{Reduction to a Bandit Instance.}  We model the scalar-threshold problem using the price discretization approach in~\cite{KleinbergL}. We discretize the continuous parameter range $[\tau_{\min}, \tau_{\max}]$ into grid points of width $\delta$, treating each point as an arm. This creates $n = O((\tau_{\max} - \tau_{\min})/\delta)$ total arms. Playing the arm corresponds to choosing that discretized value of $\tau$, yielding reward $Y(x,\tau)$ and cost $C(x,\tau)$ whose values depend on the (possibly unobserved) context $x$. Assuming the expected cost function $c(\tau) = \mathbb E_{x \sim \mathcal D}[C(x,\tau)]$ is Lipschitz continuous with parameter $c = O(1)$, rounding any feasible continuous threshold up to the nearest grid point preserves validity by monotonicity and loses at most an additive $O(\delta T)$ in cost. Applying Theorem~\ref{thm:bandit2} yields an efficiency regret of $\widetilde{O}(\sqrt{T/\delta})$. Choosing $\delta = T^{-1/3}$ balances the two regret terms, achieving an efficiency regret bound of $\widetilde{O}(T^{2/3})$.

\begin{corollary}[Scoring Rules with Bandit Feedback]
\label{cor:score_discrete}
Under the above assumptions, for the scalar-threshold scoring-rule setting, there is a policy via discretizing the scores using $\delta = T^{-1/3}$, which satisfies the pathwise conformal validity guarantee of Theorem~\ref{thm:valid2} and gives expected efficiency regret $\widetilde O(T^{2/3})$,  with constants depending on the problem parameters.
\end{corollary}

\paragraph{Multi-dimensional Scoring Rules.} The multi-armed bandit model is quite versatile. In addition to the settings above, it also captures multi-dimensional scoring rules. For instance, a ride-sharing platform may not possess a single, perfectly calibrated ``match probability'' score for each driver. Instead, candidates may be filtered based on a set of independent criteria, such as a maximum spatial distance ($\tau_{d}$) and a minimum driver quality rating ($\tau_{r}$). Such a setting can also be reduced to a bandit instance via discretizing the 2D space of parameters $(\tau_{d}, \tau_{r})$.

\subsubsection{Interval Selection}
\label{sec:interval_instantiation}

We next show how the same finite-arm model captures online interval selection without assuming a score function. In online learning of a conformal interval~\cite{Srinivas26}, the input is a sequence of 1-D points $y_t \in [0,1]$. The goal each step is to output an interval $I_t \subseteq [0,1]$ before the new point $y_t$ is realized, obtaining feedback only about whether $y_t \in I_t$, but not the point $y_t$ itself. We need to design a policy that is conformally valid, that is, the point lies in the interval $\phi$ fraction of time, and which has minimum volume $|I_t|$. We can map it to our setting by discretizing the set of intervals and treating each interval as an arm. 
This setting generalizes the setting in~\cite{Srinivas26} to bandit feedback on the reward. In both settings, the interval costs are known upfront, though we can generalize to costs being presented online as well.

Suppose we discretize the set of all possible intervals on $[0,1]$ as follows: For a parameter $\delta$, we round the left end-point of the interval down to the nearest multiple of $\delta$, and the right end-point up to the nearest multiple of $\delta$. This maps any continuous interval to one of $n = O(1/\delta^2)$ discrete arms. This rounding inflates the volume (cost) of any interval by at most $2\delta$, leading to a baseline additive regret of $O(\delta T)$. Applying Theorem~\ref{thm:bandit2} now yields a regret bound of $\widetilde{O}(\sqrt{T}/\delta)$. Balancing these two regret terms by setting $\delta = T^{-1/4}$ results in a stochastic efficiency regret of $\widetilde{O}(T^{3/4})$ for this setting.

\begin{corollary}[Online Interval Selection with Bandit Feedback]
\label{cor:interval_discrete}
For the online interval selection setting, under the above assumptions, there is a policy that proceeds via discretizing intervals using $\delta = T^{-1/4}$, which satisfies the pathwise conformal validity guarantee of Theorem~\ref{thm:valid2}, and gives expected efficiency regret $\widetilde O(T^{3/4})$.
\end{corollary}

%% file: ICLR/algorithm.tex
\subsection{The Boundary-Rule Stabilized Primal-Dual ACI Algorithm}
We present a primal-dual algorithm inspired by  bandits literature~\citep{BwK,AgarwalD,pmlr-v247-guo24a}. Set $\lambda_1 = 0$. The algorithm below is conformally valid for any $\Lambda > 0$; we will use a specific choice of $\Lambda = \frac{C_{\max}}{1-\phi}$ in our efficiency proof. 

In the description below, $t_i$ will denote the number of steps arm $i$ has been played. Initially, we play each arm $i \notin \{i_{\min}, i_{\max}\}$ once, and set $t_i =1 $ for all these arms. For $i \in \{i_{\min}, i_{\max}\}$, set $t_i = \infty$, since we assume these arms have known rewards and costs, which can be viewed as having played them infinitely many time. We do not update the dual variable during this initialization phase; after initialization, we set $\lambda=0$. For each time $t > n-2$ do:
\begin{enumerate}
    \item \textbf{UCB Estimates.} Let $\mu_{i,t_i}$ and $\chi_{i,t_i}$ denote its empirical mean reward earned and mean cost spent so far respectively. Let $\Delta_{i,t_i} = \sqrt{ \frac{2 \ln(nT)}{t_i}}$.  Compute its UCB estimates~\cite{auer} as 
    $$\bar{R}_{i, t-1} = \mu_{i,t_i} + \Delta_{i,t_i} \qquad \mbox{and} \qquad \bar{C}_{i,t-1} = \chi_{i,t_i} - C_{\max} \cdot \Delta_{i,t_i}.$$ 
    \item \textbf{Primal Selection:} If $\lambda_t \ge \Lambda$, then $i_t = i_{\max}$, else if $\lambda_t \le 0$, then $i_t = i_{\min}$, else $i_t = \arg\min_{i } \left( \bar{C}_{i,t-1} - \lambda_t \bar{R}_{i, t-1} \right)$.
    \item \textbf{Dual Update (ACI):} Observe $Y_t = r_{i_t, t}$. Update $\lambda_{t+1}= \lambda_t + \eta (\phi - Y_t)$.
\end{enumerate}

Note that when $\lambda \le  0$, arm $i_{\min}$ is always chosen so that the drift in $\lambda$ is always positive. Similarly, when arm $i_{\max}$ is chosen, the drift in $\lambda$ is always negative. This means $\lambda \in [-\eta,\Lambda+\eta]$ always. Our regret bound will depend polynomially on $\Lambda$. We justify the choice of $\Lambda$ in Section~\ref{app:pdproof}.

Before presenting the formal analysis, we emphasize the critical role of our boundary rule (setting $i_t = i_{\max}$ or $i_t = i_{\min}$ in Step 2). A direct adaptation of stochastic bandit algorithms~\cite{Cayci_Zheng_Eryilmaz_2022,pmlr-v247-guo24a} would simply play the $\mbox{argmin}$ optimal action always, but project $\lambda_t$ onto $[0,\Lambda]$ to ensure bounded dual variables. However, as noted earlier, this projection breaks the telescoping sum required for proving validity; see the discussion after Theorem~\ref{thm:valid2}. This is not merely a theoretical nuisance; Section~\ref{sec:main_exp} empirically demonstrates how a projected algorithm permanently loses track of long-term validity under distribution shifts. Our boundary rule simulates this projection with a sub-optimal boundary action, which in turn necessitates a different efficiency proof in Theorem~\ref{thm:bandit2}, which addresses the time spent at the boundary.

\subsection{Adversarial Conformal Validity}
The following theorem is a corollary of~\cite{GibbsC}, and shows a worst-case validity error of $O(\Lambda/\sqrt{T})$ assuming $n = o(T)$. Our main contribution is not this result; it is in showing that the update rule guaranteeing such validity also preserves stochastic efficiency.

\begin{theorem}
\label{thm:valid2}
For any reward sequence and any interval $[\ell,\ell+L-1] \subseteq [n-1,T]$, the Primal-Dual ACI algorithm satisfies
\[
\left|\frac{1}{L} \sum_{t=\ell}^{\ell+L-1} Y_t - \phi \right|
\le \frac{\Lambda + 2\eta}{\eta L}
= O\!\left(\frac{\Lambda}{\eta L}\right).
\]
\end{theorem}

\begin{proof}
Let $u = \ell + L$. By the update rule $\lambda_{t+1} = \lambda_t + \eta(\phi - Y_t)$, we have
\[
\lambda_u - \lambda_\ell
= \sum_{t=\ell}^{u-1} (\lambda_{t+1} - \lambda_t)
= \eta \sum_{t=\ell}^{u-1} (\phi - Y_t).
\]
Rearranging gives
\[
\frac{1}{L} \sum_{t=\ell}^{\ell+L-1} Y_t
= \phi - \frac{\lambda_u - \lambda_\ell}{\eta L}.
\]
Since $\lambda_t \in [-\eta, \Lambda+\eta]$ for all $t$, we have
$|\lambda_u - \lambda_\ell| \le \Lambda + 2\eta$.
Substituting yields the result.
\end{proof}

As a corollary, for $\eta = 1/\sqrt{T}$, we have
\[
\left| \frac{1}{L} \sum_{t=\ell}^{\ell+L-1} Y_t - \phi \right|
\le \frac{\Lambda \sqrt{T} + 2}{L}.
\]
Thus for $L = \omega(\sqrt{T})$, the average validity error is $o(1)$.
In particular, for $L = T$, the cumulative deviation
$\left|\sum_{t=1}^T Y_t - \phi T\right|$ is $O(\sqrt{T})$. In the typical setting where $n = o(L)$, the above theorem easily extends to intervals that are subsets of the entire horizon $[1,T]$ instead of $[n+1,T]$.

\subsubsection{Failure of Projected Dual Updates}
\label{sec:failure}
The exact telescoping identity above is the reason we avoid the projected dual updates that
are standard in constrained bandits. To see this, consider the natural projected variant
\[
    \lambda_{t+1}
    =
    \Pi_{[0,\Lambda]}\bigl(\lambda_t+\eta(\phi-Y_t)\bigr).
\]
Define the projection residual
$
    \epsilon_t
    =
    \lambda_{t+1}
    -
    \bigl(\lambda_t+\eta(\phi-Y_t)\bigr).
$
Then
$
    \lambda_{t+1}
    =
    \lambda_t+\eta(\phi-Y_t)+\epsilon_t,
$
and summing over an interval $[\ell,u)$ gives
\[
    \frac{1}{L}\sum_{t=\ell}^{u-1}Y_t-\phi
    =
    -\frac{\lambda_u-\lambda_\ell}{\eta L}
    +
    \frac{1}{\eta L}\sum_{t=\ell}^{u-1}\epsilon_t,
    \qquad L=u-\ell.
\]
The first term is harmless because the projected dual variable is bounded. However, the residual
term need not be small: each projection event can have $|\epsilon_t|=O(\eta)$, and therefore
$
    \frac{1}{\eta L}\sum_{t=\ell}^{u-1}|\epsilon_t|
$
can be $\Theta(1)$ if projection occurs on a constant fraction of the interval, which is possible if the sequence is adversarial. Thus projection may hide a linear amount of coverage debt in the residuals, even though the dual variable itself remains bounded. In contrast, our update is unprojected, so $\epsilon_t\equiv 0$ and the ACI telescoping identity is exact. The boundary actions in Step~2 stabilize the state through the chosen sub-optimal actions rather than by modifying the update.

\paragraph{A concrete bandit instance.}
The residual term in the projected dual update can be linear in a standard bandit instance. We spell this out for the three-arm environment. We also experimentally demonstrate the failure for this instance in Section~\ref{sec:main_exp}.

Fix a target $\phi\in(0,1)$, stepsize $\eta>0$, cap $\Lambda>1$, and cost $\gamma\in(0,1)$. There are three arms. Arm $0$ is the guaranteed-failure arm, with cost $0$ and success $0$. Arm $1$ is the guaranteed-success arm, with cost $1$ and success $1$. Arm $b$ is a cheap trap arm, with cost $\gamma$ and success $Z_t\in\{0,1\}$, where the sequence $(Z_t)$ may change adversarially over time.

Consider the natural projected analog of the primal-dual rule as described above: it maintains $\lambda_t\in[0,\Lambda]$, chooses an arm $i_t\in\arg\min_i\{\widehat c_{i,t}-\lambda_t\widehat p_{i,t}\}$, observes $Y_t=r_{i_t,t}$, and updates $\lambda_{t+1}=\Pi_{[0,\Lambda]}(\lambda_t+\eta(\phi-Y_t))$. The deterministic boundary arms are known exactly, so $(\widehat c_{0,t},\widehat p_{0,t})=(0,0)$ and $(\widehat c_{1,t},\widehat p_{1,t})=(1,1)$. The trap cost is known, so $\widehat c_{b,t}=\gamma$, and assuming the arm has been played sufficiently, its success estimate $\widehat p_{b,t}$ is approximately its empirical success mean from previous plays of arm $b$. The UCB corrections can only strengthen the argument: the boundary arms are evaluated at their exact values, while for the trap arm the optimistic estimates satisfy $\overline C_{b,t}\le \gamma$ and $\overline R_{b,t}\ge \widehat p_{b,t}$, so its UCB Lagrangian score is no larger than $\gamma-\lambda_t\widehat p_{b,t}$.

Let $I=\{\ell,\ell+1,\ldots,\ell+L-1\}$ be an adversarial ``outage'' interval. Suppose that at the beginning of the interval, the trap arm has been observed a large number $N$ of times, all of them successes, and no previous observed failures. Let $\rho=1-(1-\gamma)/\Lambda$, and assume $N/(N+L)>\rho$. During the interval, the adversary sets $Z_t=0$ for every $t\in I$. The lemma below show that on $I$, the projected dual algorithm has constant per-step validity error, and hence does not satisfy pathwise validity.

\begin{lemma}[Three-arm failure of projected ACI]
Under the setup above, the projected rule never plays the guaranteed-success arm on $I$. Hence $Y_t=0$ for all $t\in I$, so the validity error $D_I=\sum_{t\in I}(\phi-Y_t)$ equals $\phi L$. Moreover, if $m=\lceil(\Lambda-\lambda_\ell)/(\eta\phi)\rceil\le L$, then after the first $m$ steps of $I$, the projected dual variable is stuck at $\Lambda$, and every later round has ACI update error $\epsilon_t = -\eta\phi$.
\end{lemma}

\begin{proof}
Let $\widehat p_t=\widehat p_{b,t}$ denote the trap arm's empirical success estimate at the beginning of round $t$. We first show that $\widehat p_t>\rho$ throughout the interval. At the start of $I$, the trap arm has $N$ observed successes and no observed failures. During $I$, the trap arm can be played at most $L$ additional times, and every such play is a failure because $Z_t=0$ on $I$. Therefore, at every round $t\in I$, the trap estimate satisfies $\widehat p_t\ge N/(N+L)>\rho$.

Since $\Lambda>1$ and $\gamma<1$, we have $\rho>\gamma$. Indeed, $\rho-\gamma=(1-\gamma)(1-1/\Lambda)>0$. Fix any round $t\in I$. The projected update guarantees $\lambda_t\in[0,\Lambda]$. The Lagrangian scores of the three arms are $S_0=0$, $S_b \le \gamma-\lambda_t\widehat p_t$, and $S_1=1-\lambda_t$, with smaller score preferred.

If $\lambda_t<1$, then $S_0=0<1-\lambda_t=S_1$, so the guaranteed-success arm is not a minimizer. If $\lambda_t\ge 1$, then the trap arm strictly beats the guaranteed-success arm. To see this, $S_b<S_1$ is equivalent to $\lambda_t(1-\widehat p_t)<1-\gamma$. This holds because $\lambda_t\le\Lambda$ and $\widehat p_t>\rho$, so $\lambda_t(1-\widehat p_t)<\Lambda(1-\rho)=1-\gamma$. Thus in every possible case, arm $1$ is not a minimizer. 

Therefore the projected rule plays only arm $0$ or arm $b$ on $I$. Hence $Y_t=0$ for every $t\in I$. It follows immediately that $D_I=\sum_{t\in I}(\phi-Y_t)=\phi L$.

It remains to verify the statement about the cap and the residuals. Since $Y_t=0$ throughout $I$, the projected update reduces on this interval to $\lambda_{t+1}=\min\{\Lambda,\lambda_t+\eta\phi\}$. By induction, for every integer $s\in\{0,1,\ldots,L\}$, $\lambda_{\ell+s}=\min\{\Lambda,\lambda_\ell+s\eta\phi\}$. Therefore, if $m=\lceil(\Lambda-\lambda_\ell)/(\eta\phi)\rceil\le L$, then $\lambda_t=\Lambda$ for all $t\in\{\ell+m,\ldots,\ell+L\}$. 

For any round $t\in\{\ell+m,\ldots,\ell+L-1\}$, the unprojected tentative update would be $\Lambda+\eta\phi$, while the projected update returns $\Lambda$. Since $\epsilon_t=\lambda_{t+1}-(\lambda_t+\eta(\phi-Y_t))$, this gives $\epsilon_t=-\eta\phi$ on every capped round. 
\end{proof}

The condition $N/(N+L)>\rho$ simply says that the pre-shift period was long enough that the trap arm's empirical success rate remains high throughout the outage, even if every trap play during the outage fails. For any desired outage length $L$, this can be ensured using $N>\rho L/(1-\rho)$ pre-shift trap successes.

In this example, the projected rule fails because once the dual variable is capped, further failures no longer increase the state. The algorithm can therefore continue to prefer the cheap trap arm, whose estimate is still inflated by its pre-shift history. In the boundary-stabilized algorithm, reaching the upper boundary instead triggers a play of the guaranteed-success arm. Those successes are actual observations, so they ensure conformal validity rather than hiding it inside the conformal error $\epsilon_t$.

\subsection{Stochastic Efficiency}
\label{sec:full_stoch}
For showing stochastic efficiency, we assume the rewards $\mathbf{r}_t$ and costs $\mathbf{c}_t$ are drawn i.i.d. from an unknown distribution $\mathcal D$ with means $p_i = \mathbf{E}[r_{i,t}]$ and $\omega_i = \mathbf{E}[c_{i,t}]$. We assume that for arms $i \neq i_{\min}, i_{\max}$, we have $p_i \in (0,1)$ and $\omega_i \in (0,C_{\max})$. Let $C^*$ be the minimum expected cost per step achievable by a stochastic policy with per-step validity at least $\phi$ in expectation, and let $x^*$  be the corresponding randomized selection of arms:
\begin{equation} \label{eq:lp}
    C^* = \min_{x \in \Delta_n} \sum_{i=1}^n x_i \omega_i \quad \text{s.t.} \quad \sum_{i=1}^n x_i p_i \ge \phi.
\end{equation}
We will treat this optimal policy as the benchmark, and consider regret against $C^*$. Note that as is standard in the constrained bandit literature~\cite{BwK}, because any per-sequence valid policy trivially satisfies validity in expectation, the optimal cost of the baseline lower-bounds the expected cost of any per-sequence valid policy, leading to a stronger benchmark.

We will set $\Lambda = \frac{C_{\max}}{1-\phi}$. Our main result shows a sub-linear bound on the efficiency regret, where we have assumed that $n = o(T)$ and (for convenience of stating the bound that) $C_{\max} \ge 1$.

\begin{theorem}\label{thm:bandit2}
When $\eta = 1/\sqrt{T}$, the Primal-Dual algorithm achieves regret:
\begin{equation*}
    \mathbf{E} \left[ \sum_{t=1}^T (c_{i_t} - C^*) \right] \le   O\left(\frac{C_{\max}}{1-\phi}  \sqrt{nT \log (nT)}\right).
\end{equation*}
\end{theorem}

 The proof proceeds via using the squared dual variable as a Lyapunov function and bounding its drift. It gets involved because of the need to handle the boundary cases in Step 2, and our bound depends inversely on $1-\phi$, the price we pay for adversarial validity relative to  bandit regret bounds in~\cite{pmlr-v247-guo24a,BwK}.

%% file: ICLR/pdproof.tex
\subsubsection{Proof of Theorem~\ref{thm:bandit2}}
\label{app:pdproof}
Our proof uses the following claim about UCB estimates whose proof follows from a standard application of Hoeffding's inequality and union bounds, and is implicit in~\cite{auer}.

\begin{claim}[~\cite{auer}] \label{claim:auer}
With probability $1-1/T$, simultaneously for all arms $i$ and all number of pulls $s \le T$, we have
\begin{equation} \label{eq:ucb}
\mu_{i,s} - \Delta_{i,s} \le p_i \le \mu_{i,s} + \Delta_{i,s} \qquad \mbox{and} \qquad  \chi_{i,s} - C_{\max} \cdot \Delta_{i,s} \le \omega_i \le \chi_{i,s} + C_{\max} \cdot \Delta_{i,s} 
\end{equation}
\end{claim}

\begin{proof}[Proof of Theorem~\ref{thm:bandit2}]
The initial step of playing each arm once loses an additive $n C_{\max}$ to the regret that we absorb into the bounds below.  We analyze the dual potential $V_t = \lambda_t^2$. The ACI update $\lambda_{t+1} = \lambda_t + \eta (\phi - Y_t)$ implies:
\begin{equation*}
    \mathbf{E}[V_{t+1} - V_t \mid \mathcal{F}_{t-1}] = \mathbf{E}[2\eta \lambda_t (\phi - Y_t) + \eta^2 (\phi - Y_t)^2 \mid \mathcal{F}_{t-1}] \le 2\eta \lambda_t (\phi - p_{i_t}) + \eta^2.
\end{equation*}
Summing from $t=1$ to $T$ and telescoping, and assuming $\lambda_1 = 0$ (so that $V_1 = 0$), we have:
\begin{equation} 
\label{eq:drift}
    \mathbf{E} \left[ \sum_{t=1}^T \lambda_t (p_{i_t} - \phi) \right] \le \frac{\mathbf{E}[V_1 - V_{T+1}]}{2\eta} + \frac{T \eta}{2} \le  \frac{T \eta}{2}.
\end{equation}

We split time into steps $T_H = \{t | \lambda_t \ge \Lambda\}$, and $T_L = [1,T] \setminus T_H$, and develop separate bounds on Lyapunov drift for both. From Eq~(\ref{eq:drift}), we have
\begin{equation} 
\label{eq:drift2}
    \mathbf{E} \left[ \sum_{t \in T_L} \lambda_t (p_{i_t} - \phi) \right] + \mathbf{E} \left[ \sum_{t \in T_H} \lambda_t (p_{i_t} - \phi) \right] \le  \frac{T \eta}{2}.
\end{equation}

\noindent {\bf Case 1: $t \in T_L$.} Let $\mathcal E$ denote the event in Claim~\ref{claim:auer} holds. We will show
\begin{equation} \label{eq:lagrangian_step}
    \omega_{i_t} - C^* \le \lambda_t (p_{i_t} - \phi) +  2 \left(|\lambda_t| + C_{\max}\right) \cdot \Delta_{i_t,t_{i_t}} + \mathbf{1}_{\mathcal E^c} \left(C_{\max} + \Lambda\right).
\end{equation}

Note that in the event $\mathcal E^c$, we have $\omega_{i_t} - C^* \le C_{\max}$, and $\lambda_t (p_{i_t} - \phi) \ge -\Lambda$, so that the above inequality trivially holds. We therefore assume $\mathcal E$ holds.

First assume $\lambda_t > 0$. For $t \in T_L$, by the optimality of the selection rule, for all arms $i$:
\begin{equation*}
    \bar{C}_{i_t,t-1} - \lambda_t \bar{R}_{i_t, t-1} \le \bar{C}_{i,t-1}  - \lambda_t \bar{R}_{i, t-1}.
\end{equation*}
From Eq~(\ref{eq:ucb}), we have
$ \bar{R}_{i, t-1} = \mu_{i,t_i} + \Delta_{i,t_i} \ge p_i$ and $\bar{C}_{i,t-1} = \chi_{i,t_i} - C_{\max} \cdot \Delta_{i,t_i} \le \omega_i.$
Plugging these bounds into the previous inequality and noting $\lambda_t \ge 0$, we have for all arms $i$:
$$ \bar{C}_{i_t,t-1} - \lambda_t \bar{R}_{i_t, t-1} \le \omega_i - \lambda_t p_i.$$
Using Eq~(\ref{eq:ucb}) for arm $i_t$, we have
$$ \bar{R}_{i_t, t-1} = \mu_{i_t,t_{i_t}} + \Delta_{i_t,t_{i_t}} \le p_{i_t} + 2 \Delta_{i_t,t_{i_t}} $$
and
$$\bar{C}_{i_t,t-1} = \chi_{i_t,t_{i_t}} - C_{\max} \cdot \Delta_{i_t,t_{i_t}} \ge \omega_{i_t} - 2 C_{\max} \cdot \Delta_{i_t,t_{i_t}}.$$
Substituting the bounds into the previous inequality, we finally have:
\begin{equation*}
    \omega_{i_t} - \lambda_t p_{i_t} \le \omega_i - \lambda_t p_i + 2 \left(\lambda_t + C_{\max}\right) \cdot \Delta_{i_t,t_{i_t}}  \qquad \forall i.
\end{equation*}
Let $x^*$ be the optimal primal solution to Eq~(\ref{eq:lp}). Taking the expectation over $x^*$ on the RHS, and noting that $\sum_i x^*_i \omega_i \le C^*$ and $\sum_i x^*_i p_i \ge \phi$, we have:
\begin{equation*} 
    \omega_{i_t} - C^* \le \lambda_t (p_{i_t} - \phi) +  2 \left(|\lambda_t| + C_{\max}\right) \cdot \Delta_{i_t,t_{i_t}}.
\end{equation*}
This shows Eq~(\ref{eq:lagrangian_step}) when $\lambda_t > 0$.

Next, when $\lambda_t \le 0$, the algorithm always plays $i_{\min}$, meaning $\omega_{i_t} = 0$ and $p_{i_t} = 0$. In this case, Eq~(\ref{eq:lagrangian_step}) evaluates to $-C^* \le -\lambda_t \phi + 2(|\lambda_t| + C_{\max})\Delta_{i_t, t_{i_t}}$. Since $C^* \ge 0$ and $-\lambda_t > 0$, this inequality holds trivially. Therefore, Eq~(\ref{eq:lagrangian_step}) holds for all $t \in T_L$.

\noindent {\bf Case 2: $t \in T_H$.} We have $\lambda_t \ge \Lambda > 0$ and $p_{i_t} = 1 > \phi$, so that
$
\Lambda(1 - \phi) \le  \lambda_t (p_{i_t} - \phi)
$. 
Summing this over $t \in T_H$ and substituting into Eq~(\ref{eq:drift2}) and using $\eta = 1/\sqrt{T}$, we  obtain:
\begin{equation} \label{eq:dual_drift} 
\mathbf{E} \left[ \sum_{t \in T_L} \lambda_t (p_{i_t} - \phi) \right] + \mathbf{E}[|T_H|] \cdot \Lambda(1 - \phi) \le  \frac{\sqrt{T}}{2}. 
\end{equation}

\noindent {\bf Bounding the Regret.}
We now bound the total cost regret as follows:
\begin{align*}
    \mbox{Regret} & = \mathbf{E} \left[ \sum_{t =1}^T (c_{i_t} - C^*) \right] 
      = \mathbf{E} \left[ \sum_{t \in T_L} (c_{i_t} - C^*) \right] + \mathbf{E} \left[ \sum_{t \in T_H} (c_{i_t} - C^*) \right] \\
    & \le \mathbf{E} \left[ \sum_{t \in T_L} (c_{i_t} - C^*) \right] + C_{\max} \cdot \mathbf{E}[|T_H|] \\
    &\le \mathbf{E} \left[ \sum_{t \in T_L} \lambda_t (p_{i_t} - \phi) \right] + 2  \left(\Lambda + C_{\max} \right) \mathbf{E} \left[ T \cdot \mathbf{1}_{\mathcal E^c} + \sum_{i=1}^n \sum_{t_i=1}^{T_i} \sqrt{ \frac{2 \ln(nT)}{t_i}} \right] + C_{\max} \cdot \mathbf{E}[|T_H|] \\
    &\le  \frac{\sqrt{T}}{2} - \mathbf{E}[|T_H|] \cdot \Lambda(1 - \phi) +  \Lambda  \cdot O\left(\sum_{i=1}^n \sqrt{T_i \log (nT)}\right) + C_{\max} \cdot \mathbf{E}[|T_H|]  \\
    & = \frac{\sqrt{T}}{2}+ O\left(\Lambda \sqrt{nT \log (nT)} \right) + \mathbf{E}[|T_H|] \cdot \left(C_{\max} - \Lambda (1-\phi) \right) = O\left(\Lambda \sqrt{nT \log (nT)} \right).
\end{align*}
Here, the first inequality follows since $c_{i_t} \le C_{\max}$ always and $C^* \ge 0$; the second inequality follows by summing Eq~(\ref{eq:lagrangian_step}) over $T_L$ (observing that $t_{i}$ increases by $1$ whenever arm $i$ is played) and using $|\lambda_t| \le \Lambda$ for $t \in T_L$; the third inequality follows by substituting Eq~(\ref{eq:dual_drift}), using $\Lambda = \frac{C_{\max}}{1-\phi} \ge C_{\max}$, and $\Pr[\mathcal E^c] \le 1/T$. Note that since \(i_t\) and the events \(t\in T_L,T_H\) are \(\mathcal F_{t-1}\)-measurable, and the cost vector at time \(t\) is independent of \(\mathcal F_{t-1}\), we have
$
\mathbb E[\mathbf 1_{t\in T_L}(c_{i_t,t}-C^*)]
=
\mathbb E[\mathbf 1_{t\in T_L}(\omega_{i_t}-C^*)].
$
The penultimate equality follows  since $\sum_{i=1}^n T_i = T$ always, so that $\sum_{i=1}^n \sqrt{T_i} \le \sqrt{nT}$ by Jensen's inequality; and the final equality follows since $\Lambda (1-\phi) =  C_{\max}$, so that $C_{\max} - \Lambda (1-\phi) \le 0$. 
\end{proof}

%% file: ICLR/experiment.tex
\subsection{Empirical Study} 
\label{sec:main_exp}
Our main goal in the paper is to show algorithms that achieve provable theoretical bounds on efficiency regret for adversarial conformal coverage in bandit  settings. Nevertheless, we present a set of simulations on synthetic and real data of the presented algorithms.  Our goal in the experiment below is mainly to empirically verify the two core claims in the paper -- that these algorithms achieve path-wise conformal validity and sublinear stochastic efficiency, and the unprojected ACI update is needed for validity. 

\subsubsection{Interval Selection} 
We consider the interval selection setting in Section~\ref{sec:discrete}. We fix $\phi=0.8$ and discretization parameter $\delta=0.05$, so that the set of discretized intervals is fixed. The input points are drawn from ${\tt Beta}(2,5)$ on $[0,1]$. At each step, the only feedback the algorithm obtains is whether the point lies in the chosen interval. Each discretized interval is treated as a bandit arm, and we implement the algorithm in Section~\ref{sec:full}.

For this fixed grid, we compute the stochastic benchmark $C^*$ exactly from the LP in Eq.~\eqref{eq:lp}, using the true success probabilities of the discretized intervals. We run the algorithm over multiple seeds and horizons, with learning rates $\eta=a/\sqrt{T}$ for $a\in\{1,5,20\}$. The results are shown in Figure~\ref{fig:ablation_eta}. The left panel shows convergence of the cumulative success rate to $\phi$, the middle panel shows cumulative efficiency regret relative to $C^*$, and the right panel shows final regret across horizons on log-log axes with fitted empirical slopes. Since $\delta$ is fixed in this experiment, this scaling corresponds to the finite-arm theorem, where the predicted dependence is $\widetilde O(\sqrt{T})$ up to logarithmic factors. Increasing $\eta$ accelerates validity convergence, but increases resource consumption.

\begin{figure}[htbp]
    \centering
    \includegraphics[width=\textwidth]{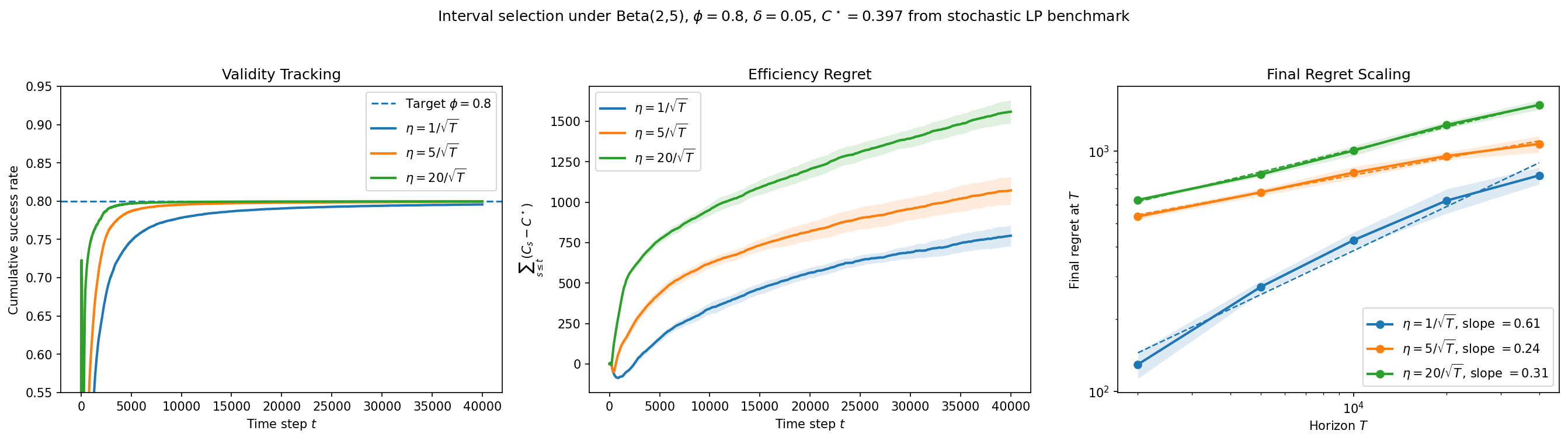}
    \caption{Interval-selection experiment under ${\tt Beta}(2,5)$ inputs with $\phi=0.8$ and $\delta=0.05$.}
    \label{fig:ablation_eta}
\end{figure}

Figure~\ref{fig:projected-validity-cost} compares our boundary-stabilized algorithm with a projected primal-dual baseline adapted from stochastic bandit literature~\cite{pmlr-v247-guo24a} that projects the dual variable into $[0,\Lambda]$ and always plays the Lagrangian optimum arm. We use the same stochastic interval-selection setting, with Beta$(2,5)$ inputs, $\phi=0.8$, $\delta=0.05$, and $\eta=20/\sqrt{T}$. The projected baseline has lower cumulative cost, but also under-covers: its cumulative success rate stays below the target $\phi$. Thus the apparent cost improvement comes from failing the validity requirement. This illustrates why replacing projection with boundary actions is important for preserving ACI-style validity, even in the stochastic setting.

\begin{figure}[htbp]
    \centering
    \includegraphics[width=0.85\linewidth]{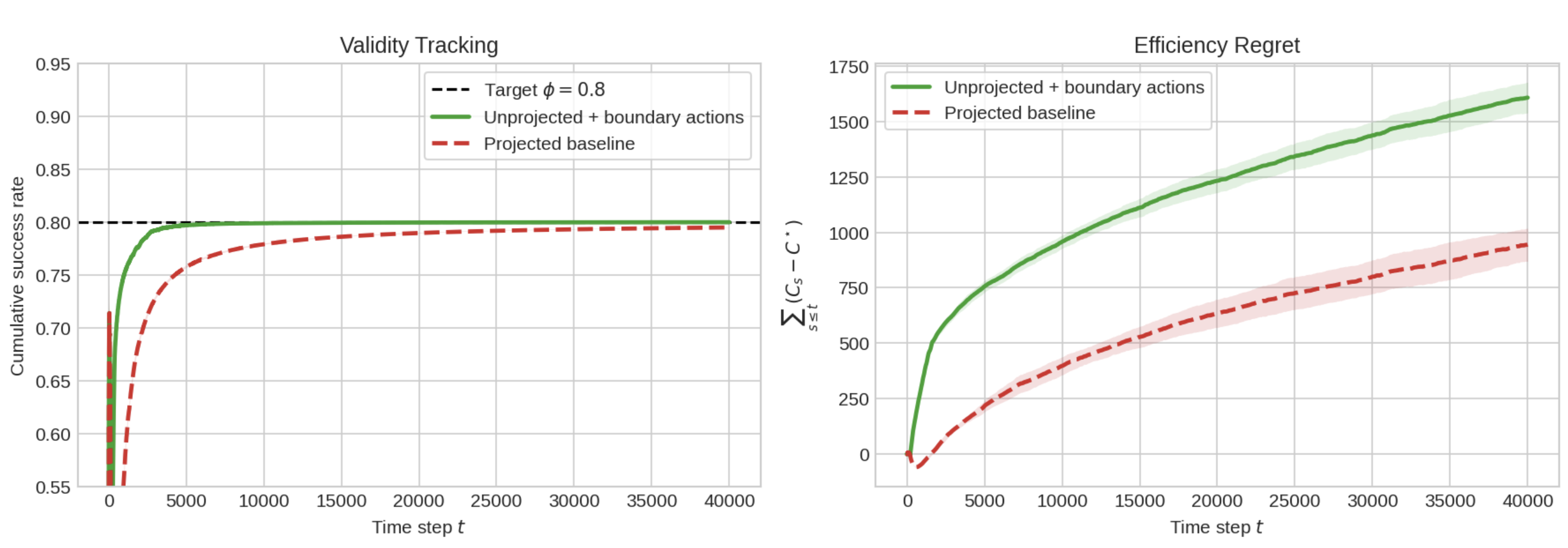}
    \caption{Projected baseline diagnostic in the stochastic interval-selection experiment.}
    \label{fig:projected-validity-cost}
\end{figure}

\subsubsection{Projection versus Boundary-Stabilized ACI} 
We now stress-test the unprojected dual update with the boundary rule (Section~\ref{sec:full}) under adversarial distribution shifts, and show that while the projected baseline ``forgets'' the coverage requirement, our algorithm gracefully recovers coverage. We first show this on a small synthetic setting where the contrast is stark, and then on a more realistic setting using CIFAR-100 data.

\paragraph{Experiment for setting in Section~\ref{sec:failure}.} We first simulate the adversarial shift on the $n = 3$ arm instance from Section~\ref{sec:failure} to visually demonstrate the failure of projected dual updates. We compare our algorithm against the projected baseline mentioned above. Here, we set $\Lambda = C_{\max}/(1-\phi) = 2$ for target $\phi=0.5$. We also set $\eta = 0.05$ and the time horizon to $T = 2000$. We employ a three-arm environment: a guaranteed safe arm ($c=1$), a cheap ``trap'' arm ($c=0.05$), and a zero-cost arm that always yields reward $0$. The trap arm yields successes ($Y = 1$), but in the time interval $[800,1300]$, an adversary forces it to fail ($Y = 0$).

\begin{figure}[htbp]
    \centering
    \includegraphics[width=0.75\textwidth]{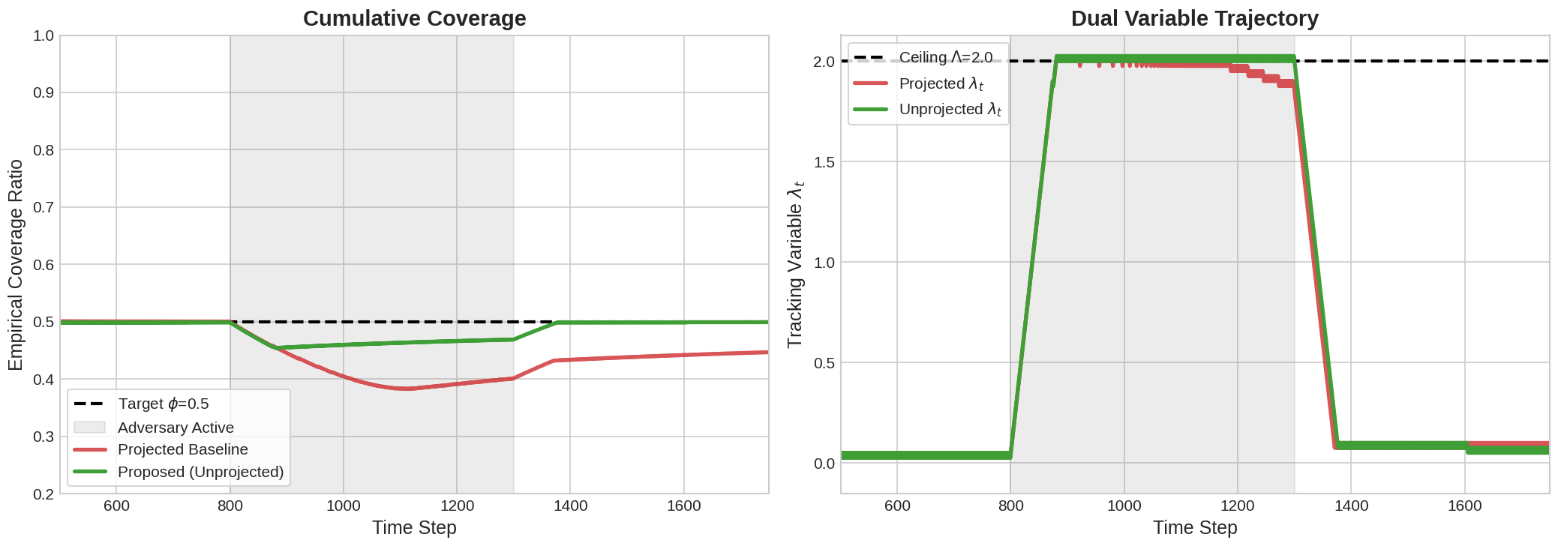}
    \caption{Comparison of projected update vs. boundary rules under an adversarial shift with $\phi = 0.5$.}
    \label{fig:ablation_proj}
\end{figure}

As shown in Fig.~\ref{fig:ablation_proj}, when the adversary strikes, both dual variables rise. The projected baseline  hits the ceiling $\Lambda=2$ and is capped. Because $\lambda_t$ cannot exceed $\Lambda$, it suffers a prolonged sequence of failures from playing the trap arm, whose long term reward still appears good. As a result, even when the environment resets, it cannot recover from this coverage failure. In contrast, for the entire duration when the dual variable takes value $\lambda = 2$, our algorithm plays $i_{\max}$ whose reward is $1$, quickly correcting for the conformal failure, and tracking validity over the time horizon. This confirms that projection destroys adversarial validity, whereas performing unprojected updates with a boundary rule guarantees it.

\paragraph{Experiment on CIFAR-100 dataset.}
We next present a semi-real experiment using CIFAR-100 top-$k$ classification. The goal of this experiment to illustrate the pathwise validity mechanism under a distribution shift using a real prediction-set construction. We use a fixed pretrained classifier to rank the $100$ CIFAR-100 labels for each image. At each time $t$, the learner chooses an arm $k_t \in \{0,5,10,20,50,100\}$, corresponding to returning the top-$k_t$ labels under this classifier, and observes only the success bit
$Y_t = \mathbf{1}\{\mbox{the true label is among the top-}k_t\mbox{ labels}\}$.
The learner does not observe the true label or its rank. The cost is $k_t/100$. 

We induce a distribution shift by ordering real CIFAR-100 examples according to the rank of their true label under the fixed classifier. In the first phase, we sample examples whose true label has rank at most $5$, so that the cheap top-$5$ arm is valid. In the second phase, we sample examples whose true label has rank at least $80$, so that all arms except the guaranteed-success arm $k=100$ fail. Thus the shift is adversarial, but the success/failure outcomes are generated by real classifier rankings on real images. We compare our boundary-stabilized unprojected ACI algorithm with the projected-dual baseline, using target $\phi=0.8$ in Figure~\ref{fig:cifar-stress}.  In the left panel, the projected-dual baseline suffers a long local validity failure after the shift, while boundary-stabilized ACI quickly returns to the target success rate. The middle panel shows the corresponding cumulative coverage debt, $\phi t-\sum_{s=1}^t Y_s$, accumulated by the projected algorithm. In contrast, the unprojected ACI update keeps this debt near zero. The right panel shows the mechanism behind the failure. Before the shift, both methods use the cheap top-$5$ arm. After the shift, boundary-stabilized ACI quickly moves to the guaranteed-success arm ($k = 100$) often enough to maintain validity, while the projected-dual baseline remains near the formerly successful cheap arm for much longer.

\begin{figure}[htbp]
    \centering
    \includegraphics[width=\linewidth]{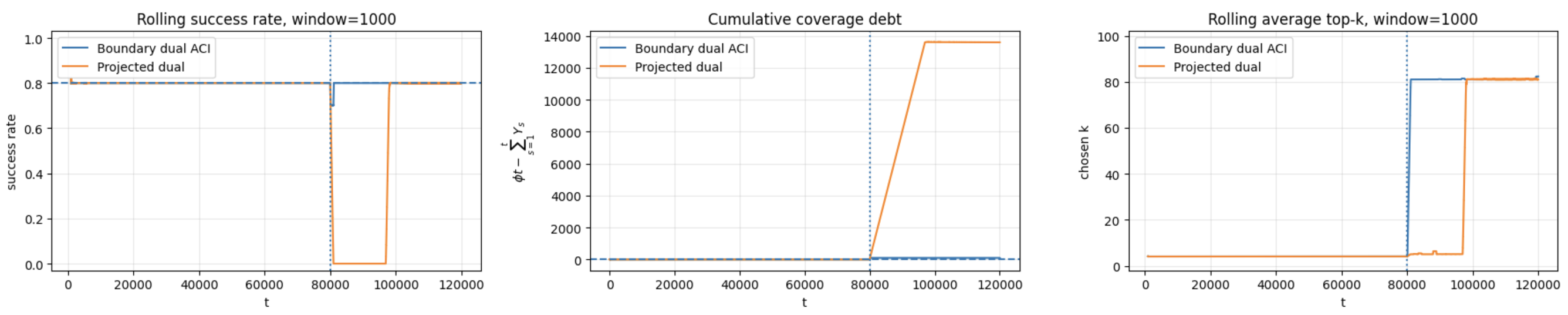}
    \caption{Comparison of projected vs. boundary rules under an adversarial shift on CIFAR-100. 
    }
    \label{fig:cifar-stress}
\end{figure}

%% file: ICLR/conclusion.tex
\section{Conclusion}
\label{sec:conclusion}
We conclude with several open questions. First, it would be interesting to extend the results in Section~\ref{sec:full}  to the case of contextual bandits with varying contexts. Further, it would be interesting to explore whether a stochastic optimality in the limit guarantee via decaying step sizes, 
analogous to the ``best of both worlds'' guarantees for conformal coverage in~\cite{angelopoulos24a}.  Finally, our results can be viewed as addressing adversarial conformal validity for simple MDPs with bandit feedback. It would be interesting to explore the integration of adversarial conformal validity with general Markov Decision Processes, and our duality-based framework provides a promising direction to address this.

%% file: neurips/bandit.tex
\section{Scoring Rules with Scalar Threshold: A Lightweight Algorithm}
\label{sec:bandit}
We now present a simpler ACI algorithm for the scalar scoring rule setting in Section~\ref{sec:fixed}. Unlike the main finite-arm algorithm, which observes both success and cost feedback and competes with the randomized stochastic benchmark, the controller here applies ACI directly to the primal threshold and observes only the success bit.  This leads to a clean ACI implementation where the only ``memory'' the algorithm carries across time is the conformal parameter $\tau$ itself. We show that this simple scheme also achieves pathwise validity and sublinear cost regret, but only under an additional ``non-flatness'' assumption on the rewards.  The regret bound is against an optimal fixed threshold policy and the dependence on $T$ is worse than in Theorem~\ref{thm:bandit2}; however, the algorithm uses less feedback than in that section. 


\paragraph{Setting.} Recall from Section~\ref{sec:fixed} that we have a black-box algorithm {\sc ALG}$(x,\tau)$ that takes as input a context $x \in \mathcal{X}$ and a scalar $\tau$.  It  outputs a success indicator $Y(x,\tau)$, so that either $Y(x,\tau) = 0$ (failure) or $Y(x,\tau) = 1$ (success). We assume this function is monotone in $\tau$ for any fixed $x$, so that there exists $\tau_x \in (\tau_{\min},\tau_{\max})$ such that $Y(x,\tau) = 0$ for $\tau < \tau_x$, and $Y(x,\tau) = 1$ for $\tau \ge \tau_x$. Therefore, $Y(x,\tau_{\max}) = 1$ (guaranteed success) and $Y(x,\tau_{\min}) = 0$ (guaranteed failure). Let $Q = \tau_{\max} - \tau_{\min}$.  There is a cost function $C$, so that the algorithm spends cost budget $C(x, \tau)$ when given input context $x$ and parameter $\tau$. We assume that for any fixed $x \in \mathcal X$, the function $C(x, \tau)$ is monotonically non-decreasing in  $\tau$.

There is an online sequence of contexts $x_t$. At each step $t$, we need to choose a parameter $\tau_t$ and run {\sc ALG}$(x_t, \tau_t)$, where $x_t$ itself could be hidden from the calibrating policy; the policy only obtains $Y(x_t,\tau_t)$ as feedback. The goal is to achieve adversarial validity against $\phi \in (0,1)$, and stochastic efficiency regret.

\subsection{The Primal ACI Algorithm} 
The calibration algorithm is a simple adaptation of the ACI procedure in~\cite{GibbsC}, applied directly to the calibration parameter instead of a dual variable. For parameter $\eta > 0$ to be chosen later, set
$$ \tau_{t+1} \leftarrow \tau_t + \eta \cdot (\phi - Y(x_t,\tau_t)).$$
Note that $\tau_t$ stays naturally bounded within $[\tau_{\min}-\eta, \tau_{\max}+\eta]$ because the drift is always positive at $\tau_{\min}$ ($Y(x_t,\tau_{\min}) = 0 < \phi$) and negative at $\tau_{\max}$ ($Y(x_t,\tau_{\max}) = 1 > \phi$). 

For notational convenience, we extend the definitions of
${\sc ALG}$, $Y$, and $C$ to all real values of $\tau$ by clipping the
threshold to $[\tau_{\min},\tau_{\max}]$ before running the algorithm.
That is, if $\tau\le \tau_{\min}$, then {\sc ALG}$(x,\tau)$ means
{\sc ALG}$(x,\tau_{\min})$; if $\tau\ge \tau_{\max}$, then
{\sc ALG}$(x,\tau)$ means {\sc ALG}$(x,\tau_{\max})$; and otherwise it
has its original meaning. The same convention applies to
$Y(x,\tau)$ and $C(x,\tau)$, and hence to the population functions
$r(\tau)=\mathbb{E}[Y(x,\tau)]$ and $c(\tau)=\mathbb{E}[C(x,\tau)]$.
The internal state $\tau_t$ itself is not projected; only the threshold
passed to {\sc ALG} is clipped. With this convention, the update is
always written as
$\tau_{t+1}=\tau_t+\eta(\phi-Y(x_t,\tau_t))$.

\paragraph{Conformal Validity.} 
The conformal validity proof is analogous to Theorem~\ref{thm:valid2}.
\begin{theorem}[Validity] \label{thm:valid}
For any  $[\ell,u] \subseteq [1,T]$ with $u - \ell = L$, we have $
    \left| \frac{1}{L} \sum_{t=\ell}^{u-1} Y_t - \phi \right| = \frac{|\tau_{u} - \tau_{\ell}|}{\eta L}.$
\end{theorem}

As a corollary, for the setting of $\eta = \frac{1}{\sqrt{T}}$ used in our efficiency guarantee below, since $|\tau_{u} - \tau_{\ell}| \le Q + 2 \eta$, we have:
$  \left| \frac{1}{L} \sum_{t=\ell}^{u-1} Y_t - \phi \right| \le \frac{Q \sqrt{T} + 2}{L} . $
For $L = \omega(\sqrt{T})$, the validity error is sub-linear in $L$.


\subsection{Efficiency of Cost Budget} 
We now assume $x_t$ is drawn i.i.d. from an unknown distribution $\mathcal D$ for all time steps $t$. Let $c(\tau) = \mathbf E_{x \sim \mathcal D} [C(x, \tau)]$ and $r(\tau) = \mathbf E_{x \sim \mathcal D} [Y(x, \tau)]$. By assumption, these functions are monotonically non-decreasing in $\tau$. In the stochastic setting, the natural benchmark is to compare against the optimal fixed-threshold policy that knows $\mathcal D$, and minimizes expected cost, while satisfying validity in expectation.  Formally, we assume $r(\tau)$ is continuous in $\tau$ and \(r(\tau_{\min})\le \phi\le r(\tau_{\max})\),  so that there is a $\tau^* \in [\tau_{\min}, \tau_{\max}]$ such that $r(\tau^*) = \phi$. If there are many such $\tau$, we define $\tau^*$ as the smallest such value. Let $C^* = c(\tau^*)$. Since the contexts are $i.i.d.$, the policy that always chooses threshold $\tau^*$ is the optimal single-threshold policy that minimizes expected cost while preserving validity in expectation every step. Now define the expected cost regret as
$$R(T) = \mathbf E\left[\sum_{t=1}^T (c(\tau_t) - C^*)^+\right],$$
the positive regret against the expected cost of a stochastic fixed-threshold policy that satisfies expected validity at least $\phi.$ This is a bit stronger than the standard notion of regret, in that bounding this suffices to bound standard regret.


\paragraph{Assumptions.} For our analysis, we assume the functions $r(\tau), c(\tau)$ are such that for constants $c_1, c_2 > 0$:
\begin{enumerate}
    \item \textbf{Success Margin:} For all $\tau > \tau^*$, $r(\tau) - \phi \ge c_1(\tau - \tau^*)$. That is, the expected success probability is strictly increasing as we increase the threshold beyond the optimal point.
    \item \textbf{Cost Lipschitz:} For all $\tau > \tau^*$, $c(\tau) - c(\tau^*) \le c_2(\tau - \tau^*)$. This ensures the cost does not jump discontinuously when we slightly over-provision the threshold, and the reduction in Section~\ref{sec:discrete} to convert the problem to apply Theorem~\ref{thm:bandit2} also uses the same assumption.
\end{enumerate}

The main new assumption over Section~\ref{sec:full} is the Success Margin or ``non-flatness'' assumption on expected reward. It essentially says that the expected reward is strictly increasing around $\tau^*$, and is a more global version of the  ``non-flatness'' assumption on the CDF of the success probability in~\cite{ge2025stochastic}. For instance, in the typical case, $r(\tau)$ is increasing and concave (diminishing returns) for $\tau > \tau^*$. In this case, it is easy to see that one setting of $c_1$ is
$\frac{r(\tau_{\max})-r(\tau^*)}{\tau_{\max} - \tau^*} \ge \frac{1-\phi}{\tau_{\max}},$ which is a constant when $\phi < 1$.  

Note that since the algorithm does not receive feedback on cost, we need an assumption on the reward function in order to achieve bounded regret. Consider for instance if the reward function does not satisfy the Success Margin assumption and is flat around $\tau^*$, say in the interval $[\tau_{\ell}, \tau_h]$. Then, $\tau_t$ could converge to a $\tau_h > \tau^*$, leading to a constant regret in cost per step if $c(\tau)$ is linearly increasing in $\tau$. Therefore, Success Margin is a natural assumption needed in this setting for sub-linear regret. 

Note that Section~\ref{sec:full} does not require the Success Margin assumption since it utilizes feedback on the cost and works with a discrete set of arms.

\paragraph{Efficiency Regret.} Our main result is the following sub-linear regret bound on the cost budget,  and our contribution is to show that these techniques not only yield conformal validity bounds as shown in~\cite{angelopoulos24a}, but also yield efficiency bounds under suitable assumptions.

\begin{theorem}\label{thm:efficient_cont} \label{thm:aci-eff}
For the Primal ACI procedure with $\eta = 1/\sqrt{T}$, we have
$ R(T) \le O(T^{3/4}),$
where the $O(\cdot)$ notation hides constants that depend polynomially on $Q, c_1, c_2$.
\end{theorem}
\begin{proof}
Define the one-sided Lyapunov potential $V_t = ((\tau_t - \tau^*)^+)^2$. We analyze the expected change $\mathbf{E}[V_{t+1} - V_t]$ by considering the position of $\tau_t$:

\paragraph{Case 1: $\tau_t \le \tau^* + \eta$.} 
In this case $V_t \ge 0$ and $V_{t+1} \le 4 \eta^2$, since $\tau_{t+1} \le \tau_t + \eta \le \tau^* + 2\eta$. Therefore, $\mathbf{E}[V_{t+1} - V_t \mid \mathcal{F}_{t-1}] \le 4 \eta^2$.

\paragraph{Case 2: $\tau_t > \tau^* + \eta$.} 
Here $V_t = (\tau_t - \tau^*)^2$. Under the update $\tau_{t+1} = \tau_t + \eta(\phi - Y_t)$, we have:
\begin{align*}
    V_{t+1} &\le ((\tau_t - \tau^*) + \eta(\phi - Y_t))^2 = V_t + 2\eta(\tau_t - \tau^*)(\phi - Y_t) + \eta^2(\phi - Y_t)^2.
\end{align*}
Taking the conditional expectation with respect to $\mathcal{F}_{t-1}$, and using $(\phi - Y_t)^2 \le 1$:
\begin{equation*}
    \mathbf{E}[V_{t+1} - V_t \mid \mathcal{F}_{t-1}] \le 2\eta(\tau_t - \tau^*)(\phi - r(\tau_t)) + \eta^2.
\end{equation*}
Using the Success Margin assumption $r(\tau_t) - \phi \ge c_1(\tau_t - \tau^*)$, we obtain:
\begin{equation*}
    \mathbf{E}[V_{t+1} - V_t \mid \mathcal{F}_{t-1}] \le -2c_1 \eta (\tau_t - \tau^*)^2 + \eta^2 = -2c_1 \eta V_t + \eta^2.
\end{equation*}

Combining both cases, we have $\mathbf{E}[V_{t+1} - V_t] \le -2c_1 \eta \mathbf{E}[V_t \cdot \mathbf{1}_{\tau_t > \tau^* + \eta}] + 4 \eta^2$. Summing over $T$,
\begin{equation*}
    \mathbf{E}[V_{T+1} - V_1] \le -2c_1 \eta \mathbf{E}\left[\sum_{t: \tau_t > \tau^* + \eta} V_t \right] + 4 T \eta^2.
\end{equation*}
Rearranging, noting $V_{T+1} \ge 0$ and $V_1 \le Q^2$, and using $\eta = 1/\sqrt{T}$:
\begin{equation} \label{eq:sum_sq}
  \mathbf{E} \left[ \sum_{t: \tau_t > \tau^* + \eta} V_t \right] = \mathbf{E} \left[ \sum_{t: \tau_t > \tau^* + \eta} (\tau_t - \tau^*)^2\right] \le \frac{Q^2}{2c_1 \eta} + \frac{4T \eta}{2c_1} = O(\sqrt{T}).
\end{equation}

We finally have $R(T) \le \mathbf{E} \left[\sum_{t: \tau_t \ge \tau^*} (c(\tau_t) - c(\tau^*)) \right]$. Using the Cost Lipschitz assumption:
\begin{align*}
    R(T) & \le c_2 \mathbf{E} \left[\sum_{t: \tau_t \ge \tau^*} (\tau_t - \tau^*) \right] \le c_2 \left( \eta T +  \mathbf{E}\left[ \sum_{t: \tau_t > \tau^* + \eta} (\tau_t - \tau^*) \right] \right) \\ & = c_2 \left(\sqrt{T} + \mathbf{E}\left[ \sum_{t: \tau_t > \tau^* + \eta} (\tau_t - \tau^*) \right] \right).
\end{align*}
By applying Cauchy-Schwarz inequality followed by Jensen's inequality:
\begin{equation*}
    \mathbf{E} \left[ \sum_{t: \tau_t > \tau^* + \eta} (\tau_t - \tau^*) \right] \le \mathbf{E}\left[\sqrt{T \cdot \sum_{t: \tau_t > \tau^* + \eta} \left(\tau_t - \tau^* \right)^2} \right] \le \sqrt{T \cdot \mathbf{E} \left[ \sum_{t: \tau_t > \tau^* + \eta} (\tau_t - \tau^*)^2 \right]} .
\end{equation*}
Substituting the bound from (\ref{eq:sum_sq}) and combining with the previous inequality:
\begin{equation*}
    R(T) \le  c_2 \left( \sqrt{T} + \sqrt{T \cdot O(\sqrt{T})} \right) = O(T^{3/4}).
\end{equation*}
This completes the proof.
\end{proof}

\paragraph{Comparison with Theorem~\ref{thm:bandit2}.} Note that for this setting, Section~\ref{sec:discrete} shows a regret bound of $O(T^{2/3})$. The difference is that there, we assume the policy gets feedback about not just the reward $Y(x_t,\tau_t) \in \{0,1\}$ earned during step $t$, but also the cost $C(x_t,\tau_t)$ spent at step $t$. Though that bound was better than the regret bound of $O(T^{3/4})$ shown here, that algorithm requires maintaining a large set of discretized arms and associated UCB scores, while the algorithm above is ``memoryless'' and stores its entire state in the parameter $\tau_t$. Furthermore, the above algorithm does not need information about the cost at all, while the algorithm in Section~\ref{sec:full} needs to observe feedback on both the reward and cost.  However, the regret guarantee above requires the Success Margin condition and is with respect to the best fixed-threshold policy. In general, this guarantee cannot be strengthened to the fully randomized stochastic optimum over thresholds. Theorem~\ref{thm:bandit2} does not have these limitations.

%% file: main.bbl
\begin{thebibliography}{31}
\providecommand{\natexlab}[1]{#1}
\providecommand{\url}[1]{\texttt{#1}}
\expandafter\ifx\csname urlstyle\endcsname\relax
  \providecommand{\doi}[1]{doi: #1}\else
  \providecommand{\doi}{doi: \begingroup \urlstyle{rm}\Url}\fi

\bibitem[Agrawal and Devanur(2014)]{AgarwalD}
Shipra Agrawal and Nikhil~R. Devanur.
\newblock Bandits with concave rewards and convex knapsacks.
\newblock In \emph{Proceedings of the Fifteenth ACM Conference on Economics and
  Computation}, EC '14, page 989–1006, New York, NY, USA, 2014. Association
  for Computing Machinery.
\newblock \doi{10.1145/2600057.2602844}.
\newblock URL \url{https://doi.org/10.1145/2600057.2602844}.

\bibitem[Angelopoulos et~al.(2023)Angelopoulos, Candes, and
  Tibshirani]{angelopoulos2023conformal}
Anastasios Angelopoulos, Emmanuel Candes, and Ryan Tibshirani.
\newblock Conformal {PID} control for time series prediction.
\newblock In A.~Oh, T.~Naumann, A.~Globerson, K.~Saenko, M.~Hardt, and
  S.~Levine, editors, \emph{Advances in Neural Information Processing Systems},
  volume~36, pages 23047--23074. Curran Associates, Inc., 2023.
\newblock URL
  \url{https://proceedings.neurips.cc/paper_files/paper/2023/file/47f2fad8c1111d07f83c91be7870f8db-Paper-Conference.pdf}.

\bibitem[Angelopoulos and Bates(2021)]{angelopoulos2021gentle}
Anastasios~N. Angelopoulos and Stephen Bates.
\newblock A gentle introduction to conformal prediction and distribution-free
  uncertainty quantification.
\newblock \emph{arXiv preprint arXiv:2107.07511}, 2021.
\newblock URL \url{https://arxiv.org/abs/2107.07511}.

\bibitem[Angelopoulos et~al.(2025)Angelopoulos, Jordan, and
  Tibshirani]{angelopoulos2025gradient}
Anastasios~N. Angelopoulos, Michael~I. Jordan, and Ryan~J. Tibshirani.
\newblock Gradient equilibrium in online learning: Theory and applications.
\newblock \emph{Journal of Machine Learning Research}, 26\penalty0
  (305):\penalty0 1--68, 2025.
\newblock URL \url{http://jmlr.org/papers/v26/25-0356.html}.

\bibitem[Angelopoulos et~al.(2024)Angelopoulos, Barber, and
  Bates]{angelopoulos24a}
Anastasios~Nikolas Angelopoulos, Rina Barber, and Stephen Bates.
\newblock Online conformal prediction with decaying step sizes.
\newblock In Ruslan Salakhutdinov, Zico Kolter, Katherine Heller, Adrian
  Weller, Nuria Oliver, Jonathan Scarlett, and Felix Berkenkamp, editors,
  \emph{Proceedings of the 41st International Conference on Machine Learning},
  volume 235 of \emph{Proceedings of Machine Learning Research}, pages
  1616--1630. PMLR, 21--27 Jul 2024.
\newblock URL \url{https://proceedings.mlr.press/v235/angelopoulos24a.html}.

\bibitem[Areces et~al.(2025)Areces, Mohri, Hashimoto, and
  Duchi]{areces2025online}
Felipe Areces, Christopher Mohri, Tatsunori Hashimoto, and John Duchi.
\newblock Online conformal prediction via online optimization.
\newblock In \emph{Forty-second International Conference on Machine Learning},
  2025.
\newblock URL \url{https://openreview.net/forum?id=KwGc2JUIDK}.

\bibitem[Auer et~al.(2002)Auer, Cesa-Bianchi, and Fischer]{auer}
P.~Auer, N.~Cesa-Bianchi, and P.~Fischer.
\newblock Finite-time analysis of the multi-armed bandit problem.
\newblock \emph{Machine Learning}, 47:\penalty0 235--256, 2002.
\newblock URL \url{http://www2.compute.dtu.dk/pubdb/pubs/2088-full.html}.

\bibitem[Badanidiyuru et~al.(2018)Badanidiyuru, Kleinberg, and Slivkins]{BwK}
Ashwinkumar Badanidiyuru, Robert Kleinberg, and Aleksandrs Slivkins.
\newblock Bandits with knapsacks.
\newblock \emph{J. ACM}, 65\penalty0 (3), March 2018.
\newblock ISSN 0004-5411.
\newblock \doi{10.1145/3164539}.
\newblock URL \url{https://doi.org/10.1145/3164539}.

\bibitem[Barber et~al.(2023)Barber, Cand{\`e}s, Ramdas, and
  Tibshirani]{barber2023conformal}
Rina~Foygel Barber, Emmanuel Cand{\`e}s, Aaditya Ramdas, and Ryan Tibshirani.
\newblock Conformal prediction beyond exchangeability.
\newblock \emph{The Annals of Statistics}, 51\penalty0 (2):\penalty0 816--845,
  2023.

\bibitem[Cayci et~al.(2022)Cayci, Zheng, and
  Eryilmaz]{Cayci_Zheng_Eryilmaz_2022}
Semih Cayci, Yilin Zheng, and Atilla Eryilmaz.
\newblock A lyapunov-based methodology for constrained optimization with bandit
  feedback.
\newblock \emph{Proceedings of the AAAI Conference on Artificial Intelligence},
  36\penalty0 (4):\penalty0 3716--3723, Jun. 2022.
\newblock \doi{10.1609/aaai.v36i4.20285}.
\newblock URL \url{https://ojs.aaai.org/index.php/AAAI/article/view/20285}.

\bibitem[Ding et~al.(2020)Ding, Wei, Yang, Wang, and Jovanović]{ding2020}
Dongsheng Ding, Xiaohan Wei, Zhuoran Yang, Zhaoran Wang, and Mihailo~R.
  Jovanović.
\newblock Provably efficient safe exploration via primal-dual policy
  optimization, 2020.
\newblock URL \url{https://arxiv.org/abs/2003.00534}.

\bibitem[Feldman et~al.(2023)Feldman, Ringel, Bates, and Romano]{Feldman22}
Shai Feldman, Liran Ringel, Stephen Bates, and Yaniv Romano.
\newblock Achieving risk control in online learning settings.
\newblock \emph{Transactions on Machine Learning Research}, 2023.
\newblock ISSN 2835-8856.
\newblock URL \url{https://openreview.net/forum?id=5Y04GWvoJu}.

\bibitem[Gao et~al.(2025)Gao, Shan, Srinivas, and
  Vijayaraghavan]{gao2025volume}
Chao Gao, Liren Shan, Vaidehi Srinivas, and Aravindan Vijayaraghavan.
\newblock Volume optimality in conformal prediction with structured prediction
  sets.
\newblock In \emph{Forty-second International Conference on Machine Learning},
  2025.
\newblock URL \url{https://openreview.net/forum?id=oNDhnGrD51}.

\bibitem[Ge et~al.(2025)Ge, Bastani, and Bastani]{ge2025stochastic}
Haosen Ge, Hamsa Bastani, and Osbert Bastani.
\newblock Stochastic online conformal prediction with semi-bandit feedback.
\newblock In \emph{Forty-second International Conference on Machine Learning},
  2025.
\newblock URL \url{https://openreview.net/forum?id=IdRrKDsTZ8}.

\bibitem[Gibbs and Candes(2021)]{GibbsC}
Isaac Gibbs and Emmanuel Candes.
\newblock Adaptive conformal inference under distribution shift.
\newblock In M.~Ranzato, A.~Beygelzimer, Y.~Dauphin, P.S. Liang, and J.~Wortman
  Vaughan, editors, \emph{Advances in Neural Information Processing Systems},
  volume~34, pages 1660--1672. Curran Associates, Inc., 2021.
\newblock URL
  \url{https://proceedings.neurips.cc/paper_files/paper/2021/file/0d441de75945e5acbc865406fc9a2559-Paper.pdf}.

\bibitem[Guo and Liu(2024)]{pmlr-v247-guo24a}
Hengquan Guo and Xin Liu.
\newblock Stochastic constrained contextual bandits via lyapunov optimization
  based estimation to decision framework.
\newblock In Shipra Agrawal and Aaron Roth, editors, \emph{Proceedings of
  Thirty Seventh Conference on Learning Theory}, volume 247 of
  \emph{Proceedings of Machine Learning Research}, pages 2204--2231. PMLR, 30
  Jun--03 Jul 2024.

\bibitem[Gupta et~al.(2022)Gupta, Jung, Noarov, Pai, and Roth]{Gupta22}
Varun Gupta, Christopher Jung, Georgy Noarov, Mallesh~M. Pai, and Aaron Roth.
\newblock {Online Multivalid Learning: Means, Moments, and Prediction
  Intervals}.
\newblock In Mark Braverman, editor, \emph{13th Innovations in Theoretical
  Computer Science Conference (ITCS 2022)}, volume 215 of \emph{Leibniz
  International Proceedings in Informatics (LIPIcs)}, pages 82:1--82:24,
  Dagstuhl, Germany, 2022. Schloss Dagstuhl -- Leibniz-Zentrum f{\"u}r
  Informatik.
\newblock ISBN 978-3-95977-217-4.
\newblock \doi{10.4230/LIPIcs.ITCS.2022.82}.
\newblock URL
  \url{https://drops.dagstuhl.de/entities/document/10.4230/LIPIcs.ITCS.2022.82}.

\bibitem[Immorlica et~al.(2022)Immorlica, Sankararaman, Schapire, and
  Slivkins]{Immorlica22}
Nicole Immorlica, Karthik Sankararaman, Robert Schapire, and Aleksandrs
  Slivkins.
\newblock Adversarial bandits with knapsacks.
\newblock \emph{J. ACM}, 69\penalty0 (6), November 2022.
\newblock ISSN 0004-5411.
\newblock \doi{10.1145/3557045}.
\newblock URL \url{https://doi.org/10.1145/3557045}.

\bibitem[Kleinberg and Leighton(2003)]{KleinbergL}
Robert Kleinberg and Tom Leighton.
\newblock The value of knowing a demand curve: Bounds on regret for online
  posted-price auctions.
\newblock In \emph{Proceedings of the 44th Annual IEEE Symposium on Foundations
  of Computer Science}, FOCS '03, page 594, USA, 2003. IEEE Computer Society.
\newblock ISBN 0769520405.

\bibitem[Kushner and Yin(2013)]{kushner2013stochastic}
H.~Kushner and G.G. Yin.
\newblock \emph{Stochastic Approximation and Recursive Algorithms and
  Applications}.
\newblock Stochastic Modelling and Applied Probability. Springer New York,
  2013.
\newblock ISBN 9781489926968.
\newblock URL \url{https://books.google.com/books?id=sB0GCAAAQBAJ}.

\bibitem[Mahdavi et~al.(2012)Mahdavi, Jin, and Yang]{mahdavi2012trading}
Mehrdad Mahdavi, Rong Jin, and Tianbao Yang.
\newblock Trading regret for efficiency: online convex optimization with long
  term constraints.
\newblock \emph{J. Mach. Learn. Res.}, 13\penalty0 (1):\penalty0 2503–2528,
  September 2012.
\newblock ISSN 1532-4435.

\bibitem[Neely(2010)]{Neely}
Michael~J. Neely.
\newblock \emph{Stochastic Network Optimization with Application to
  Communication and Queueing Systems}.
\newblock Synthesis Lectures on Communication Networks. Morgan {\&} Claypool
  Publishers, 2010.
\newblock ISBN 978-3-031-79994-5.
\newblock \doi{10.2200/S00271ED1V01Y201006CNT007}.
\newblock URL \url{https://doi.org/10.2200/S00271ED1V01Y201006CNT007}.

\bibitem[Ramalingam et~al.(2025)Ramalingam, Kiyani, and Roth]{ramalingam2025}
Ramya Ramalingam, Shayan Kiyani, and Aaron Roth.
\newblock The relationship between no-regret learning and online conformal
  prediction.
\newblock In \emph{Forty-second International Conference on Machine Learning},
  2025.
\newblock URL \url{https://openreview.net/forum?id=JAcOYCqNo9}.

\bibitem[Slivkins et~al.(2024)Slivkins, Zhou, Sankararaman, and
  Foster]{Slivkins_context}
Aleksandrs Slivkins, Xingyu Zhou, Karthik~Abinav Sankararaman, and Dylan~J.
  Foster.
\newblock Contextual bandits with packing and covering constraints: a modular
  lagrangian approach via regression.
\newblock \emph{J. Mach. Learn. Res.}, 25\penalty0 (1), January 2024.
\newblock ISSN 1532-4435.

\bibitem[Srinivas(2025)]{Srinivas26}
Vaidehi Srinivas.
\newblock Online conformal prediction with efficiency guarantees, 2025.
\newblock URL \url{https://arxiv.org/abs/2507.02496}.

\bibitem[Tibshirani et~al.(2019)Tibshirani, Barber, Cand\`{e}s, and
  Ramdas]{tibshirani2019conformal}
Ryan Tibshirani, Rina~Foygel Barber, Emmanuel Cand\`{e}s, and Aaditya Ramdas.
\newblock \emph{Conformal prediction under covariate shift}.
\newblock Curran Associates Inc., Red Hook, NY, USA, 2019.

\bibitem[Vovk et~al.(2005)Vovk, Gammerman, and Shafer]{vovk2005algorithmic}
Vladimir Vovk, Alexander Gammerman, and Glenn Shafer.
\newblock \emph{Algorithmic Learning in a Random World}.
\newblock Springer, 2005.

\bibitem[Wang and Qiao(2024)]{WangQ}
Zhou Wang and Xingye Qiao.
\newblock Efficient online set-valued classification with bandit feedback.
\newblock In \emph{Proceedings of the 41st International Conference on Machine
  Learning}, ICML'24. JMLR.org, 2024.

\bibitem[Yang et~al.(2026)Yang, Kim, and
  Park]{yang2026onlineconformalpredictionadversarial}
Junyoung Yang, Kyungmin Kim, and Sangdon Park.
\newblock Online conformal prediction with adversarial semi-bandit feedback via
  regret minimization, 2026.
\newblock URL \url{https://arxiv.org/abs/2604.17984}.

\bibitem[Zimmert and Seldin(2019)]{Zimmert19a}
Julian Zimmert and Yevgeny Seldin.
\newblock An optimal algorithm for stochastic and adversarial bandits.
\newblock In Kamalika Chaudhuri and Masashi Sugiyama, editors,
  \emph{Proceedings of the Twenty-Second International Conference on Artificial
  Intelligence and Statistics}, volume~89 of \emph{Proceedings of Machine
  Learning Research}, pages 467--475. PMLR, 16--18 Apr 2019.
\newblock URL \url{https://proceedings.mlr.press/v89/zimmert19a.html}.

\bibitem[Zimmert et~al.(2019)Zimmert, Luo, and Wei]{ZimmertL}
Julian Zimmert, Haipeng Luo, and Chen{-}Yu Wei.
\newblock Beating stochastic and adversarial semi-bandits optimally and
  simultaneously.
\newblock \emph{CoRR}, abs/1901.08779, 2019.
\newblock URL \url{http://arxiv.org/abs/1901.08779}.

\end{thebibliography}
